\documentclass{article}

\usepackage{microtype}
\usepackage{graphicx}
\usepackage{subfigure}
\usepackage{booktabs} 

\usepackage{amsmath}
\usepackage{amssymb}
\usepackage{mathtools}
\usepackage{amsthm}
\usepackage{amsfonts}
\usepackage[textsize=tiny]{todonotes}
\usepackage{tikz}
\usepackage{mathtools}
\usepackage{comment}
\usepackage[normalem]{ulem}
\usepackage{thmtools} 
\usepackage{thm-restate}

\usepackage{hyperref}
\usepackage[capitalize,noabbrev]{cleveref}

\theoremstyle{definition}
\newtheorem{definition}{Definition}[]

\theoremstyle{plain}
\newtheorem{theorem}[definition]{Theorem}

\newtheorem{lemma}[definition]{Lemma}

\definecolor{ao}{rgb}{0.3, .7, 0.0}

\newcommand{\eee}{\color{black}}
\newcommand{\mzz}{\color{black}}

\usepackage[accepted]{icml2023}

\icmltitlerunning{Achieving High Accuracy with PINNs via Energy Natural Gradient Descent}

\begin{document}

\twocolumn[
\icmltitle{Achieving High Accuracy with PINNs via Energy Natural Gradient Descent}



\icmlsetsymbol{equal}{*}

\begin{icmlauthorlist}
\icmlauthor{Johannes M\"{u}ller}{mpi}
\icmlauthor{Marius Zeinhofer}{simula}
\end{icmlauthorlist}

\icmlaffiliation{mpi}{Max Planck Institute for Mathematics in the Sciences
Inselstraße 22, 04103 Leipzig, Germany}
\icmlaffiliation{simula}{Department of Numerical Analysis and Scientific Computing, Simula Research Laboratory, Kristian Augusts Gate 23, 0164, Oslo, Norway}

\icmlcorrespondingauthor{Marius Zeinhofer}{mariusz@simula.no}
\icmlcorrespondingauthor{Johannes M\"{u}ller}{jmueller@mis.mpg.de}

\icmlkeywords{Machine Learning, ICML}

\vskip 0.3in
]



\printAffiliationsAndNotice{}  

\begin{abstract}
We propose \emph{energy natural gradient descent}, a natural gradient method with respect to a Hessian-induced Riemannian metric as an optimization algorithm for physics-informed neural networks (PINNs) and the deep Ritz method. As a main motivation we show that the update direction in function space resulting from the energy natural gradient corresponds to the Newton direction modulo an orthogonal projection onto the model's tangent space. We demonstrate experimentally that energy natural gradient descent yields highly accurate solutions with errors several orders of magnitude smaller than what is obtained when training PINNs with standard optimizers like gradient descent, Adam or BFGS, even when those are allowed significantly more computation time. 
\mzz 
We show that the approach can be combined with deterministic and stochastic discretizations of the integral terms and with deep networks allowing for an application in higher dimensional settings.
\eee
\end{abstract}
%

\section{Introduction}

Neural network based PDE solvers have recently experienced an enormous growth in popularity and attention within the scientific community following the works of~\cite{weinan2017deep, han2018solving, sirignano2018dgm, weinan2018deep, raissi2019physics, li2021fourier}. 
In this article we focus on methods, which parametrize the solution of the PDE by a neural network and use a formulation of the PDE in terms of a minimization problem to construct a loss function used to train the network.
The works following this ansatz can be divided into the two approaches: (a) residual minimization of the PDEs residual in strong form, this is known under the name \emph{physics informed neural networks } or \emph{deep Galerkin method}, see for example~\cite{dissanayake1994neural, lagaris1998artificial, sirignano2018dgm, raissi2019physics}; (b) if existent, leveraging the variational formulation to obtain a loss function, this is known as the \emph{deep Ritz method}~\cite{weinan2018deep}, see also~\cite{beck2020overview, weinan2021algorithms} for in depth reviews of these methods. 

One central reason for the rapid development of these methods is their mesh free nature which allows easy incorporation of data and their promise to be effective in high-dimensional and parametric problems, that render mesh-based approaches infeasible. Nevertheless, in practice when these approaches are tackled directly with well established optimizers like GD, SGD, Adam or BFGS, they often fail to produce accurate solutions even for problems of small size. This phenomenon is increasingly well documented in the literature where it is attributed to an insufficient optimization leading to a variety of optimization procedures being suggested, where accuracy better than in the order of $10^{-3}$ relative $L^2$ error can rarely be achieved \cite{hao2021efficient, wang2021understanding, wang2022and, krishnapriyan2021characterizing, davi2022pso, zeng2022competitive}. The only exceptions are ansatzes, which are conceptionally different from direct gradient based optimization, more precisely greedy algorithms and a reformulation as a 
min-max game~\cite{hao2021efficient, zeng2022competitive}.

\paragraph{Contributions}
We provide a simple, yet effective optimization method that achieves high accuracy for a range of PDEs when combined with the PINN ansatz. 
Although we evaluate the approach on PDE related tasks, it can be applied to a wide variety of training problems.
Our main contributions can be summarized as follows:
\begin{itemize}
    \item We introduce the notion of \emph{energy natural gradients}. 
    This natural gradient is defined via the Hessian of the training objective in function space\mzz, see Definition~\ref{def:ENG}\eee 
    
    We show that an energy natural gradient update in
    parameter space corresponds to a Newton update in
    function space. In particular, for quadratic energies the
    function space update approximately moves into the
    direction of the error $u^\ast-u_\theta$\mzz, see Theorem~\ref{thm:main_thm}\eee.
    
    \item We demonstrate the capabilities of the energy natural gradient combined with a simple line search to achieve an accuracy, which is several orders of magnitude higher compared to standard optimizers like GD, Adam, \mzz BFGS \eee or a natural gradient defined via Sobolev inner products. 
    These examples include PINN formulations of stationary and evolutionary PDEs as well as the deep Ritz formulation of a nonlinear ODE. \mzz The numerical evaluation is contained in Section~\ref{sec:experiments}\eee. 
    
\end{itemize}

\paragraph{Related Works}
Here, we 
focus on improving the training process and thereby the accuracy of PINNs. 
It has been observed that 
the magnitude of the gradient contributions from the PDE residuum, the boundary terms and the initial conditions 
often possess imbalanced magnitudes.
To address this, different weighting strategies for the individual components of the loss have been developed~\cite{wang2021understanding,van2022optimally,wang2022and}.
Albeit improving PINN training, non of the mentioned works reports relative $L^2$ errors below $10^{-4}$. 

The choice of the collocation points in the discretization of PINN losses has been investigated in a variety of works~\cite{lu2021deepxde, nabian2021efficient, daw2022rethinking,zapf2022investigating, wang2022respecting, wu2023comprehensive}. 
Common in all these studies is the observation that collocation points should be concentrated in regions of high PDE residual and 
we refer to~\cite{daw2022rethinking, wu2023comprehensive} for an extensive comparisons of the different proposed sampling strategies in the literature. Further, for time dependent problems curriculum learning
is reported to mitigate training pathologies associated with solving evolution problems with a long time horizon~\cite{wang2022respecting, krishnapriyan2021characterizing}. Again, while all aforementioned works considerably improve PINN training, in non of the contributions errors below $10^{-4}$ could be achieved.

Different optimization strategies, which are conceptionally different to a direct gradient based optimization of the objective, have been proposed in the context of PINNs.
For instance, greedy algorithms where used to incrementally build a shallow neural neuron by neuron, which led to high accuracy, up to relative errors of $10^{-8}$, for a wide range of PDEs~\cite{hao2021efficient}. However, the proposed greedy algorithms are only computationally tractable for shallow neural networks.
Another ansatz is to reformulate the quadratic PINN loss as a saddle-point problem involving a network for the approximation of the solution and a discriminator network that penalizes a non-zero residual. The resulting saddle-point formulation can be solved with competitive gradient descent~\cite{zeng2022competitive} and the authors report highly accurate -- up to $10^{-8}$ relative $L^2$ error -- PINN solutions for a number of example problems. 
This approach however comes at the price of training two neural networks and exchanging a minimization problem for a saddle-point problem.  
Finally, particle swarm optimization methods have been proposed in the context of PINNs, where they improve over the accuracy of standard optimizers, but fail to achieve  accuracy better than $10^{-3}$ despite their computation burden~\cite{davi2022pso}.

Natural gradient methods are an established optimization algorithm and we give an overview in Section~\ref{sec:natgrad} and discuss here only works related to the numerical solution of PDEs. 
In fact, without explicitly referring to the natural gradient literature and terminology, natural gradients are used in the PDE constrained optimization community in the context of finite elements. For example, in certain situations the mass or stiffness matrices can be interpreted as Gramians, showing that this ansatz is indeed a natural gradient method. For explicit examples we refer to ~\cite{schwedes2016iteration, schwedes2017mesh}. 
In the context of neural network based approaches, a variety of natural gradients induced by Sobolev, Fisher-Rao and Wasserstein geometries have been proposed and tested for PINNs~\cite{nurbekyan2022efficient}.
This work focuses on the efficient implementation of these methods and does not consider energy based natural gradients, which we find to be necessary in order to achieve high accuracy.

\paragraph{Notation}
We denote the space of functions on \(\Omega\subseteq\mathbb R^d\) that are integrable in $p$-th power by \(L^p(\Omega)\) and endow it with its canonical norm. 
For a sufficiently smooth function $u$ we denote its partial derivatives by $\partial_i u = \partial u / \partial x_i$ and denote the tensor associated by the $l$-th derivative by $(D^lu)_{i_1, \dots, i_l} \coloneqq \partial_{i_1}\dots\partial_{i_l} u$. 
We denote the gradient of a sufficiently smooth function $u$ by $\nabla u = (\partial_1 u, \dots, \partial_d u)^\top$ and the Laplace operator $\Delta$ is defined by $\Delta u \coloneqq \sum_{i=1}^d \partial_i^2 u$.
We denote the \emph{Sobolev space} 
of functions with weak derivatives up to order \(k\) in $L^p(\Omega)$ by \(W^{k,p}(\Omega)\), which is a Banach space with the norm
\[ \lVert u\rVert_{W^{k,p}(\Omega)}^p \coloneqq \sum_{l=0}^k\lVert D^{l} u \rVert_{L^p(\Omega)}^p. \]
In the following we mostly work with the case $p=2$ and write $H^k(\Omega)$ instead of $W^{k,2}(\Omega)$.

Consider natural numbers \(d, m, L, N_0, \dots, N_L\) and let $\theta = \left((A_1, b_1), \dots, (A_L, b_L)\right)$ be a tuple of matrix-vector pairs where \(A_l\in\mathbb R^{N_{l}\times N_{l-1}}, b_l\in\mathbb R^{N_l}\) and \(N_0 = d, N_L = m\). Every matrix vector pair \((A_l, b_l)\) induces an affine linear map \(T_l\colon \mathbb R^{N_{l-1}} \to\mathbb R^{N_l}\). The \emph{neural network function with parameters} \(\theta\) and with respect to some \emph{activation function} \(\rho\colon\mathbb R\to\mathbb R\) is the function
\[u_\theta\colon\mathbb R^d\to\mathbb R^m, \quad x\mapsto T_L(\rho(T_{L-1}(\rho(\cdots \rho(T_1(x)))))).\]
The \emph{number of parameters} and the \emph{number of neurons} of such a network is given by \(\sum_{l=0}^{L-1}(n_{l}+1)n_{l+1}\)
.
We call a network \emph{shallow} if it has depth \(2\) and \emph{deep} otherwise. 
In the remainder, we restrict ourselves to the case \(m=1\) since we only consider real valued functions.
Further, in our experiments we choose $\tanh$ as an activation function in order to assume the required notion of smoothness of the network functions $u_\theta$ and the parametrization $\theta\mapsto u_\theta$.

For $A\in\mathbb R^{n\times m}$ we denote any pseudo inverse of $A$ by $A^+$.

\section{Preliminaries}

Various neural network based approaches for the approximate solution of PDEs have been suggested~\cite{beck2020overview, weinan2021algorithms, kovachki2021neural}.
Most of these cast the solution of the PDE as the minimizer of a typically convex energy over some function space and use this energy to optimize the networks parameters. We present two prominent approaches and introduce the unified setup that we use to treat both of these approaches later. 

\paragraph{Physics-Informed Neural Networks}
Consider a general partial differential equation of the form
\begin{align}
    \begin{split}
        \mathcal L u 
        & = f \quad \text{in } \Omega \\
        \mathcal B u & = g \quad \text{on } \partial\Omega,
    \end{split}
    \label{eq:Poisson}
\end{align}
where $\Omega\subseteq\mathbb R^d$ is an open set, $\mathcal L$ is a -- possibly non-linear -- partial differential operator and $\mathcal B$ is a boundary value operator. We assume that the solution $u$ is sought in a Hilbert space $X$ and that the right-hand side $f$ and the boundary values $g$ are square integrable functions on $\Omega$ and $\partial\Omega$ respectively. In this situation, we can reformulate \eqref{eq:Poisson} as a minimization problem with objective function
\begin{equation}
    E(u) =  \int_\Omega (\mathcal L u - f)^2
    \mathrm{d}x + \tau \int_{\partial\Omega} (\mathcal B u-g)^2\mathrm ds,
\end{equation}
for a penalization parameter $\tau > 0$. A function $u\in X$ solves \eqref{eq:Poisson} if and only if $E(u)=0$. In order to obtain an approximate solution, one can parametrize the function $u_\theta$ by a neural network and minimize the network parameters $\theta\in\mathbb R^p$ according to the loss function 
\begin{equation}
    L(\theta) \coloneqq \int_\Omega (\mathcal Lu_\theta - f)^2\mathrm dx + \tau \int_{\partial\Omega}\mathcal (\mathcal Bu_\theta - g)^2\mathrm ds.
\end{equation}
This general approach to formulate equations as minimization problems is known as \emph{residual minimization} and in the context of neural networks for PDEs can be traced back to~\cite{dissanayake1994neural, lagaris1998artificial}.
More recently, this ansatz was popularised under the names \emph{deep Galerkin method} or \emph{physics-informed neural networks}, where the loss can also be augmented to encorporate a regression term steming from real world measurements of the solution~\cite{sirignano2018dgm,raissi2019physics}. 
In practice, the integrals in the objective function have to be discretized in a suitable way.

\paragraph{The Deep Ritz Method}
When working with weak formulations of PDEs it is standard to consider the variational formulation, i.e., to consider an energy functional such that the Euler-Lagrange equations are the weak formulation of the PDE. This idea was already exploited by~\cite{ritz1909neue} to compute the coefficients of polynomial approximations to solutions of PDEs and popularized in the context of neural networks in~\cite{weinan2018deep} who coined the name \emph{deep Ritz method} for this approach.
Abstractly, this approach is similar to the residual formulation. Given a variational energy $E\colon X \to \mathbb R,$ on a Hilbert space $X$
one parametrizes the ansatz by a neural network $u_\theta$ and arrives at the loss function $L(\theta)\coloneqq E(u_\theta)$. 
Note that this approach is different from PINNs, for example for the Poisson equation $-\Delta u = f$, the residual energy is given by $u\mapsto\lVert \Delta u + f \rVert_{L^2(\Omega)}^2$, where the corresponding variational energy is given by $u\mapsto\frac12\lVert \nabla u\rVert_{L^2(\Omega)}^2-\int_\Omega fu\mathrm dx$. In particular, the energies require different smoothness of the functions and are hence defined on different Sobolev spaces.

Incorporating essential boundary values in the Deep Ritz Method differs from the PINN approach. 
Whereas in PINNs for any $\tau>0$ the unique minimizer of the energy is the solution of the PDE, in the deep Ritz method the minimizer of the penalized energy solves a Robin boundary value problem, which can be interpreted as a perturbed problem. In order to achieve a good approximation of the original problem the penalty parameters need to be large, which leads to ill conditioned problems~\cite{muller2022error, courte2023robin}.

\paragraph{General Setup}
Both, physics informed neural networks as well as the deep Ritz method fall in the general framework of minimizing an energy $E\colon X\to\mathbb R$ 
or more precisely the associated objective function $L(\theta) \coloneqq E(u_\theta)$ over the parameter space of a neural network. Here, we assume $X$ to be a Hilbert space of functions and the functions $u_\theta$ computed by the neural network with parameters $\theta$ to lie in $X$ and assume that $E$ admits a unique minimizer $u^\star\in X$. 
Further, we assume that the parametrization $P\colon\mathbb R^p\to X, \theta\mapsto u_\theta$ is differentiable and denote its range by $\mathcal F_\Theta=\{u_\theta:\theta\in\mathbb R^p\}$. We denote the generalized tangent space on this parametric model by 
\begin{equation}
    T_\theta \mathcal F_\Theta \coloneqq  
    \operatorname{span} \left\{\partial_{\theta_i} u_\theta : i=1, \dots, p \right\}.
\end{equation}

\paragraph{Accuracy of NN Based PDE Solvers}
Besides considerable improvement in the PINN training process, as discussed in the Section on related work, gradient based optimization of the original PINN formulation could so far not break a certain optimization barrier, even for simple situations. Typically achieved errors are of the order $10^{-3}$ measured in the $L^2$ norm. This phenomenon is attributed to the stiffness of the PINN formulation, as experimentally verified in \cite{wang2021understanding}. Furthermore, the squared residual formulation of the PDE squares the condition number -- which is well known for classical discretization approaches \cite{zeng2022competitive}. As discretizing PDEs leads to ill-conditioned linear systems, this deteriorates the convergence of iterative solvers such as standard gradient descent. On the other hand, natural gradient descent circumvents this \emph{pathology of the discretization} by guaranteeing an update direction following the function space gradient information where the PDE problem often is of a simpler structure. We refer to Theorem~\ref{thm:main_thm} and the Appendix~\ref{sec:proofs} for a rigorous explanation.

\section{Energy Natural Gradients 
}\label{sec:natgrad}
The concept of \emph{natural gradients} was popularized by Amari in the context of parameter estimation in supervised learning and blind source separation~\cite{amari1998natural}. 
The idea here is to modify the update direction in a gradient based optimization scheme to emulate gradient in a suitable representation space of the parameters.
Whereas, this ansatz was already formulated for general metrics it is usually attributed to the use of the Fisher metric on the representation space, but also products of Fisher metrics, Wasserstein and Sobolev geometries have been successfully used~\cite{kakade2001natural, li2018natural, nurbekyan2022efficient}.
After the initial applications in supervised learning and blind source separation, it was successfully adopted in reinforcement learning~\cite{kakade2001natural, peters2003reinforcement, bagnell2003covariant,morimura2008new}, inverse problems~\cite{nurbekyan2022efficient}, neural network training~\cite{schraudolph2002fast,pascanu2014revisiting, martens2020new} and generative models~\cite{shen2020sinkhorn, lin2021wasserstein}.
One sublety in the natural gradients is the definition of a geometry in the function space. This can either be done axiomatically or through the Hessian of a potential function~\cite{amari2010information, amari2016information, wang2022hessian, Mueller2022Convergence}. 
We follow the idea to work with the natural gradient induced by the Hessian of the convex function space objective in which the natural gradient can be interpreted as a generalized Gauss-Newton method which has been suggested for neural network training for supervised learning tasks~\cite{ren2019efficient, cai2019gram, gargiani2020promise, martens2020new}.
Contrary to existing works we encounter infinite dimensional and  and not strongly convex objective in our applications. 

Here, we consider the setting of the minimization of a convex energy $E\colon 
X\to\mathbb R$ defined on a Hilbert space $X$, which covers both physics informed neural networks and the deep Ritz method.
\mzz As an objective function for the optimization of the networks parameters we use $L(\theta) = E(u_\theta)$ like before. \eee
We define the \emph{Hilbert} and \emph{energy Gram matrices} by 
\begin{align}\label{eq:gramHilbert}
    G_H(\theta)_{ij} \coloneqq \langle\partial_{\theta_i} u_\theta, \partial_{\theta_j} u_\theta\rangle_X \quad 
\end{align}
and
\begin{equation}\label{eq:gramEnergy}
    G_E(\theta)_{ij} \coloneqq D^2E(u_\theta)(\partial_{\theta_i} u_\theta, \partial_{\theta_j} u_\theta).
\end{equation}
The update direction $\nabla^H L(\theta) = G_H(\theta)^+\nabla L(\theta)$ \mzz was proposed with $\langle \cdot, \cdot\rangle_{X}$ being a Sobolev inner product for neural network training~\cite{nurbekyan2022efficient}
and we refer to it as the \emph{Hilbert natural gradient} (H-NG) or in the special case the $X$ is a Sobolev space the \emph{Sobolev natural gradient}\eee. It is well known in the literature\footnote{For regular and singular Gram matrices and finite dimensional spaces see~\cite{amari2016information, van2022invariance}, an argument for infinite dimensional space can be found in the appendix.} on natural gradients that
\footnote{\mzz Here, the Hilbert space gradient $\nabla E(u)\in X$ is the unique element satisfying $\langle \nabla E(u), v\rangle_X = DE(u)v$, where $DE$ denotes the Fréchet derivative.
\eee
}
\begin{equation}
    DP_\theta\nabla^HL(\theta) = \Pi_{T_\theta \mathcal F_\Theta}( \nabla E(u_\theta)).
\end{equation}
In words, following the natural gradient amounts to moving along the projection of the Hilbert space gradient onto the model's tangent space in function space.
The observation that identifying the function space gradient via the Hessian leads to a Newton update motivates the concept of energy natural gradients that we now introduce.

\begin{definition}[Energy Natural Gradient]\label{def:ENG}
    Consider the problem 
    $\min_{\theta\in\mathbb R^p} L(\theta)$, 
    where $ 
    L(\theta)= E(u_\theta)$ and denote the Euclidean gradient by $\nabla L(\theta)$. 
    Then we call 
    \begin{equation}
        \nabla^E L(\theta) \coloneqq G_E^+(\theta)\nabla L(\theta),
    \end{equation}
    the 
    \emph{energy natural gradient (E-NG)}\footnote{Note that this is different from the \emph{energetic natural gradients} proposed in~\cite{thomas2016energetic}, which defines natural gradients based on the energy distance rather than the Fisher metric.}. 
\end{definition}

For a linear PDE operator $\mathcal L$, the residual 
yields a quadratic energy and the energy Gram matrix takes the form

\begin{align}
    \begin{split}
        G_E(\theta)_{ij} 
        &= 
        \int_\Omega \mathcal L (\partial_{\theta_i}u_\theta) \mathcal L (\partial_{\theta_j}u_\theta) \mathrm dx 
        \\
        &+ 
        \tau \int_{\partial\Omega} \mathcal B (\partial_{\theta_i}u_\theta) \mathcal B (\partial_{\theta_j}u_\theta)  \mathrm ds
    \end{split}
\end{align}

On the other hand, the deep Ritz method for a quadratic energy $E(u) = \frac12 a(u,u) - f(u)$, where $a$ is a symmetric and coercive bilinear form and $f\in X^*$ yields
\begin{equation}
    G_E(\theta)_{ij} = a(\partial_{\theta_i} u_\theta, \partial_{\theta_j}u_\theta).
\end{equation}

For the energy natural gradient we have the following result relating energy natural gradients to Newton updates. 
\begin{restatable}[Energy Natural Gradient in Function Space]{theorem}{ENGFunctionSpace}\label{thm:main_thm}
If we assume that $D^2E$ is coercive everywhere, then we have\footnote{Here, we interpret the bilinear form $D^2E(u_\theta)\colon H\times H\to\mathbb R$ as an operator $D^2E(u_\theta)\colon H\to H$; further $\Pi_{T_\theta \mathcal F_\Theta}^{D^2E(u_\theta)}$ denotes the projection with respect to the inner product defined by $D^2E(u_\theta)$.}
\begin{equation}\label{eq:pushforwardENGNewton}
    DP_\theta\nabla^EL(\theta) = \Pi_{T_\theta \mathcal F_\Theta}^{D^2E(u_\theta)}(D^2E(u_\theta)^{-1} \nabla E(u_\theta)).
\end{equation}
Assume now that $E$ is a quadratic function with bounded and positive definite 
second derivative $D^2E = a$ that admits a minimizer $u^*\in X$. then it holds that 
\begin{equation}\label{eq:pushforwardENG}
    DP_\theta\nabla^EL(\theta) = \Pi_{T_\theta \mathcal F_\Theta}^{a}( u_\theta - u^*).
\end{equation}
 
\end{restatable}
\begin{proof}[Proof idea, full proof in the appendix]
In the case that $D^2E$ is coercive, it induces a Riemannian metric on the Hilbert space $X$. 
Since the gradient with respect to this metric is given by $D^2E(u)^{-1}\nabla E(u)$ the identity~\eqref{eq:pushforwardENGNewton} follows analogously to the finite dimensional case or case of Hilbert space NGs.
In the case that the energy $E$ is quadratic and $D^2E=a$ is bounded and non degenerate, the gradient with respect to the inner product $a$ is not classically defined. 
However, one can check that $a(u-u^\ast, v) = DE(u)v$, i.e., that the error $u-u^*$
can be interpreted as a gradient with respect to the inner product $a$, which yields~\eqref{eq:pushforwardENG}. 
\end{proof}

In particular, we see from~\eqref{eq:pushforwardENGNewton} and~\eqref{eq:pushforwardENG} that using the energy NG in parameter space is closely related to a Newton update in function space, where for quadratic energies the Newton direction is given by the error $u_\theta-u^\star$.

\paragraph{Complexity of H-NG and E-NG}
The computation of the H-NG and E-NG is -- up to the assembly of the Gram matrices $G_H$ and $G_E$ -- equally expensive. 
Luckily, the Gram matrices are often equally expensive to compute.
For quadratic problems the Hessian is typically not harder to evaluate than the Hilbert inner product~\eqref{eq:gramHilbert} and even for non quadratic cases closed form expressions of~\eqref{eq:gramEnergy} in terms of inner products are often available, see~\ref{subsec:nonlinearDR}. Note that H-NG emulates GD and E-NG emulates a Newton method in $X$.
In practice, the computation of the natural gradient is expensive since it requires the solution of a system of linear equations, which has complexity $O(p^3)$, where $p$ is the parameter dimension. 
Compare this to the cost of $O(p)$ for the computation of the gradient.
In our experiments, we find that E-NGD achieves significantly higher accuracy compared to GD and Adam even when the latter once are allowed more computation time.

\section{Experiments}\label{sec:experiments}

We test the energy natural gradient approach on \mzz four \eee problems: a PINN formulation of a two-dimensional Poisson equation\mzz, a PINN formulation of a five-dimensional Poisson equation, \eee a PINN formulation of a one-dimensional heat equation and a deep Ritz formulation of a one-dimensional, nonlinear elliptic equation.

\paragraph{Description of the Method}
For all our numerical experiments, we realize an energy natural gradient step with a line search as described in Algorithm~\ref{alg:E-NGD}.
We choose the interval $[0,1]$ for the line search determining the learning rate since a learning rate of $1$ would correspond to an approximate Newton step in function space. 
However, since the parametrization of the model is non linear, it is beneficial to conduct the line search and can not simply choose the Newton step size.
In our experiments, we use a grid search over a logarithmically spaced grid on $[0,1]$ to determine the learning rate $\eta^*$. Although naive, this can easily be parallelized and performs fast and efficient in our experiments.
\begin{algorithm}
\caption{Energy Natural Gradient with Line Search}\label{alg:E-NGD}
\begin{algorithmic}
\STATE {\bfseries Input:} initial parameters $\theta_0\in\mathbb R^p$,  
$N_{max}$ 
\FOR{$k=1, \dots, N_{max}$}   
\STATE Compute $\nabla L(\theta)\in\mathbb R^p$ 
\STATE $G_E(\theta)_{ij} \gets D^2E(\partial_{\theta_i}u_\theta, \partial_{\theta_j}u_\theta)$ for $i,j =1, \dots, p$ 
\STATE $\nabla^E L(\theta) \gets G^+_E(\theta)\nabla L(\theta)$ 
\STATE $\eta^* \gets \arg\min_{\eta\in[0,1]} \,L( \theta - \eta \nabla^E L(\theta) )$ 
\STATE $\theta_k = \theta_{k-1} - \eta^* \nabla^E L(\theta)$ 
\ENDFOR
\end{algorithmic}
\end{algorithm}
The assembly of the Gram matrix $G_E$ can be done efficiently in parallel, avoiding a potentially costly loop over index pairs $(i,j)$. 
Instead of computing the pseudo inverse of the Gram matrix $G_E(\theta)$ 
we solve the least square problem 
\begin{equation}\label{eq:lsqNG}
    \nabla^E L(\theta)\in\arg\min_{\psi\in\mathbb R^p}\| G_E(\theta)\psi - \nabla L(\theta) \|^2_2.
\end{equation} 
For the numerical evaluation of the integrals appearing in the loss function as well as in the entries of the Gram matrix we \mzz experiment both with fixed integration points on a regular grid and repeatedly and randomly drawn integration points. \eee We initialize the network's weights and biases according to a Gaussian with standard deviation $0.1$ and vanishing mean.

\paragraph{\mzz Evaluation \eee}
\mzz
We report the relative\footnote{i.e., normalized by the norm of the solution} $L^2$ and $H^1$ errors during and after the optimization process. For this we use 10 times more integration points than during the optimization.
\eee
We compare the efficiency of 
energy NGs to the following optimizers. First, we consider vanilla gradient descent (denoted as GD in our experiments) with a line search on a logarithmic grid. Then, we test the performance of Adam with an exponentially decreasing learning rate schedule to prevent oscillations, where we start with an initial learning rate of $10^{-3}$ that after $1.5\cdot 10^4$ steps starts to decrease by a factor of $10^{-1}$ every $10^4$ steps until a minimum learning rate of $10^{-7}$ is reached or the maximal amount of iterations is completed. \mzz We do also compare to the quasi-Newton method BFGS \cite{nocedal1999numerical} \eee. 
Finally, we test the Hilbert natural gradient descent with line search (denoted by H-NGD).

\paragraph{Computation Details}
For our implementation we rely on the library JAX~\cite{jax2018github}, where all required derivatives are computed using JAX' automatic differentiation module. 
The JAX implementation of the least square solve  
relies on a singular value decomposition. 
For the implementation of the BFGS optimizer we rely on the implementation jaxopt.BFGS. 
All experiments were run on a single NVIDIA RTX 3080 Laptop GPU in double precision. The code to reproduce the experiments can be found in the repository \url{https://github.com/MariusZeinhofer/Natural-Gradient-PINNs-ICML23}.

\subsection{Poisson Equation}\label{sec:Poisson_2d}
We consider the two dimensional Poisson equation
\begin{equation*}
    -\Delta u (x,y) = f(x,y) = 2\pi^2\sin(\pi x) \sin(\pi y) 
\end{equation*}
on the unit square $[0,1]^2$ with zero boundary values. The solution is given by
\begin{equation*}
    u^*(x,y) = \sin(\pi x) \sin(\pi y)
\end{equation*}
and the PINN loss of the problem is
\begin{align}\label{eq:poisson_loss}
\begin{split}
    L (\theta) & = \frac{1}{N_\Omega}\sum_{i=1}^{N_\Omega}(\Delta u_\theta(x_i,y_i) + f(x_i,y_i))^2 \\ & \qquad\qquad\quad  + \frac{1}{N_{\partial\Omega}}\sum_{i=1}^{N_{\partial\Omega}}u_\theta(x^b_i,y^b_i)^2,
\end{split}
\end{align}
where $\{(x_i,y_i)\}_{i=1,\dots,N_\Omega}$ denote the interior collocation points and $\{(x^b_i,y^b_i)\}_{i=1,\dots,N_{\partial\Omega}}$ denote the collocation points on $\partial\Omega$. In this case the energy inner product on $H^2(\Omega)$ is given by
\begin{equation}\label{eq:poisson_energy_product}
    a(u,v) = \int_\Omega \Delta u \Delta v  \mathrm dx + \int_{\partial\Omega}u v \mathrm ds.
\end{equation}
Note that this inner product is not coercive\footnote{the inner product is coercive with respect to the $H^{1/2}(\Omega)$ norm, see~\cite{muller2022notes}} on $H^2(\Omega)$ and different from the $H^2(\Omega)$ inner product. 
The integrals in  
~\eqref{eq:poisson_energy_product} are computed using the same collocation points as in the definition of the PINN loss function $L$ in~\eqref{eq:poisson_loss}. To approximate the solution $u^*$ we use a shallow neural network with the hyperbolic tangent as activation function and a width of 64, thus there are 257 trainable weights. We choose 900 equi-distantly spaced collocation points in the interior of $\Omega$ and 120 collocation points on the boundary. The energy natural gradient descent and the Hilbert natural gradient descent are applied for $500$ iterations each, whereas we train for $2\cdot 10^{5}$ iterations of GD and Adam. 

\begin{figure}[h]
    \centering
    \includegraphics[width=\linewidth]{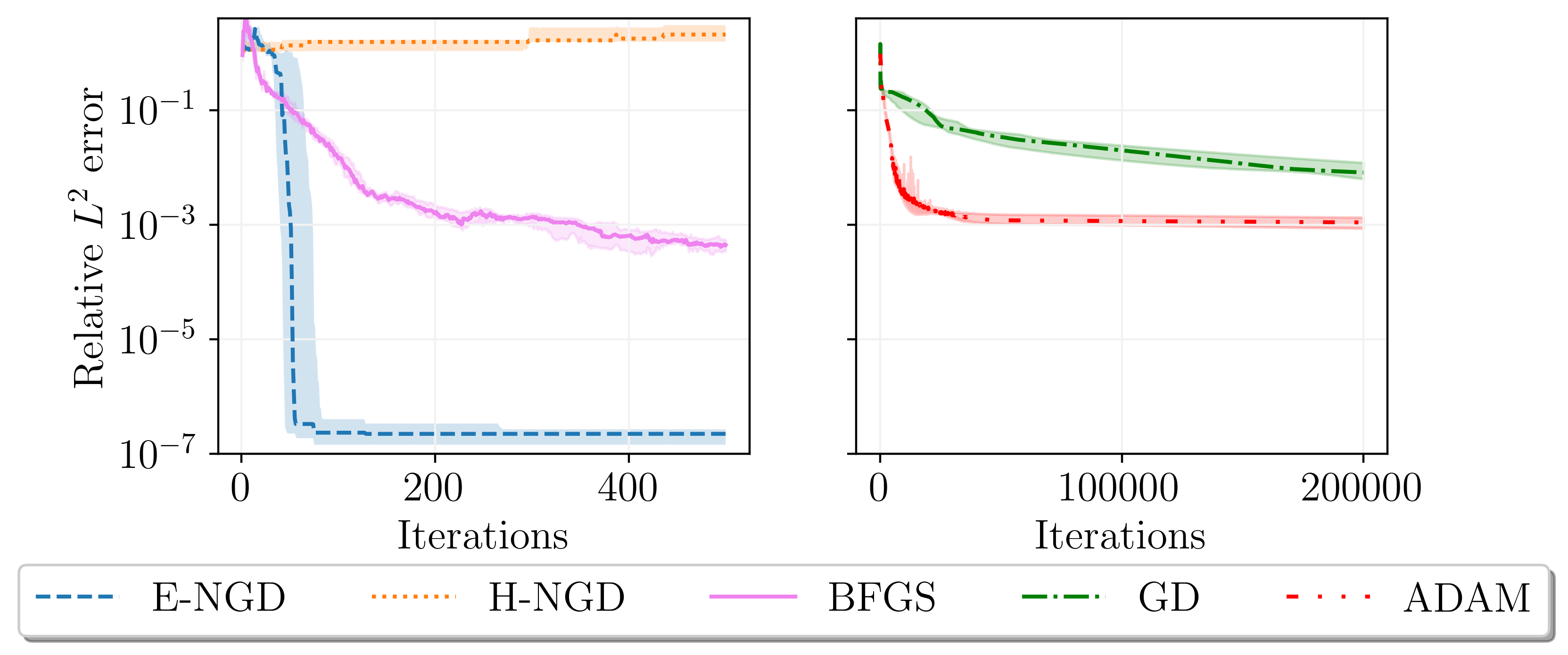}
    \vspace{-0.5cm}
    \caption{
    \mzz
    Median relative $L^2$ errors  for the two dimensional Poisson equation example over 10  
    initializations for the five optimizers
    ; 
    the shaded area denotes the region between the first and third quartile
    ; note that GD and Adam are run for $400$ times more iterations and GD, Adam and BFGS are given significantly more computation time than NGD, see~Table~\ref{table:poisson_runtimes}. \eee
    }\label{fig:poisson}
\end{figure}

\renewcommand{\arraystretch}{1.1}
\begin{table}[h]
\begin{center}
\begin{tabular}{|c|c|c|c|}
\hline
                & Median & Minimum & Maximum  \\
\hline 
GD 
& $8.2 \cdot 10^{-3}$  & $2.6 \cdot 10^{-3}$   &  $1.5 \cdot 10^{-2}$  \\
Adam & $1.1 \cdot 10^{-3}$ & $6.9 \cdot 10^{-4}$  & $1.3 \cdot 10^{-3}$  \\
H-NGD 
& $1.2$  & $4.0$  & $2.1$  \\
E-NGD  
& $\mathbf{2.4\cdot 10^{-7}}$ & $\mathbf{1.0\cdot 10^{-7}}$  & $\mathbf{4.1\cdot 10^{-7}}$  \\
BFGS 
& $4.4 \cdot 10^{-4}$ & $1.2 \cdot 10^{-4}$  & $9.6 \cdot 10^{-4}$  \\
\hline
\end{tabular}
\caption{\mzz Median, minimum and maximum of the relative $L^2$ errors for the Poisson equation example achieved by different optimizers over $10$ 
initializations. Here, energy and Hilbert NG descent and BFGS are run for $500$ and the other methods for $2\cdot10^5$ iterations.\eee}\label{table:poisson}
\end{center}
\end{table}

\begin{table}[ht]
\begin{center}
\begin{tabular}{|c|c|c|}
\hline
& Time per Iteration      & Full Optimization Time   \\
\hline 
GD     & $\mathbf{1.8\cdot 10^{-2}} \textbf{ s}$  & $1 \textrm{h}$     \\
Adam   & $3.7\cdot 10^{-2} \textrm{s}$  & $1 \text{h } 6 \textrm{min}$      \\
H-NGD  & $8.9\cdot 10^{-2} \textrm{s}$  & $44.5 \textrm{s}$      \\ 
E-NGD  & $8.6\cdot 10^{-2} \textrm{s}$  & $\mathbf{43} \textbf{s}$      \\ 
BFGS   & $1.8 \textrm{s}$               & $15 \textrm{min}$      \\ 
\hline
\end{tabular}
\caption{\mzz Computational times for the optimizers 
for the two dimensional Poisson example. For the time per iteration we averaged over 100 iterations. The full optimization time is calculated as the product of total iteration count and iteration time. The experiments were conducted on a single NVIDIA RTX 3080 Laptop GPU. \eee}\label{table:poisson_runtimes}
\end{center}
\end{table}

As reported in Table~\ref{table:poisson} and Figure~\ref{fig:poisson}, we observe that the energy NG updates require relatively few iterations to produce a highly accurate approximate solution of the Poisson equation.  Note that the Hilbert NG descent did not converge at all, stressing the importance of employing the geometric information of the Hessian of the function space objective, as is done in energy NG descent.

\mzz
The first order optimizers we consider, i.e., Adam and vanilla gradient descent
reliably decrease the relative errors, but fail to achieve an accuracy higher than $6.9\cdot 10^{-4}$ even though we allow for a much higher number of iterations. The quasi-Newton method BFGS \cite{nocedal1999numerical} achieves higher accuracy than the first order methods, however the energy natural gradient method is still roughly two orders of magnitude more accurate, compare to Table~\ref{table:poisson}.   
\eee

\mzz
With our current implementation and the network sizes we consider, one natural gradient update is only twice to three times as costly as one iteration of the Adam algorithm, compare also to Table~\ref{table:poisson_runtimes}.
Training a PINN model with optimizers such as Adam easily requires $100$ times the amount of iterations -- without being able to produce highly accurate solutions -- of what we found necessary for natural gradient training, rendering the proposed approach both faster and more accurate. Note that one optimization using the natural gradient method takes less than a minute, whereas the optimization time using the Adam optimizer takes above two hours, this is an improvement by two orders of magnitude.
\eee 

To illustrate the difference between the energy natural gradient $\nabla^E L(\theta)$, the standard gradient $\nabla L(\theta)$ and the
error 
$ 
u_\theta - u^*$, we plot the effective update directions $DP(\theta) \nabla L(\theta)$ and $DP(\theta) \nabla^{E} L(\theta)$ in function space at initialization, see Figure~\ref{fig:poisson_pushs}.  
Clearly, the energy natural gradient update direction matches the error much better than the vanilla parameter gradient.

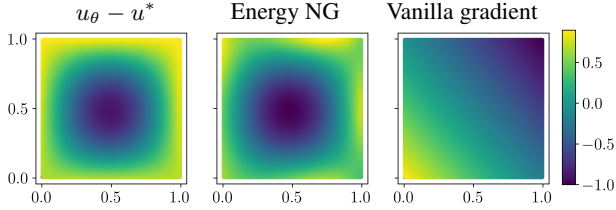
\begin{figure}[h]
    \centering
    \begin{tikzpicture}
        \node[inner sep=0pt] (r1) at (0,0)
    {\includegraphics[width=\linewidth]{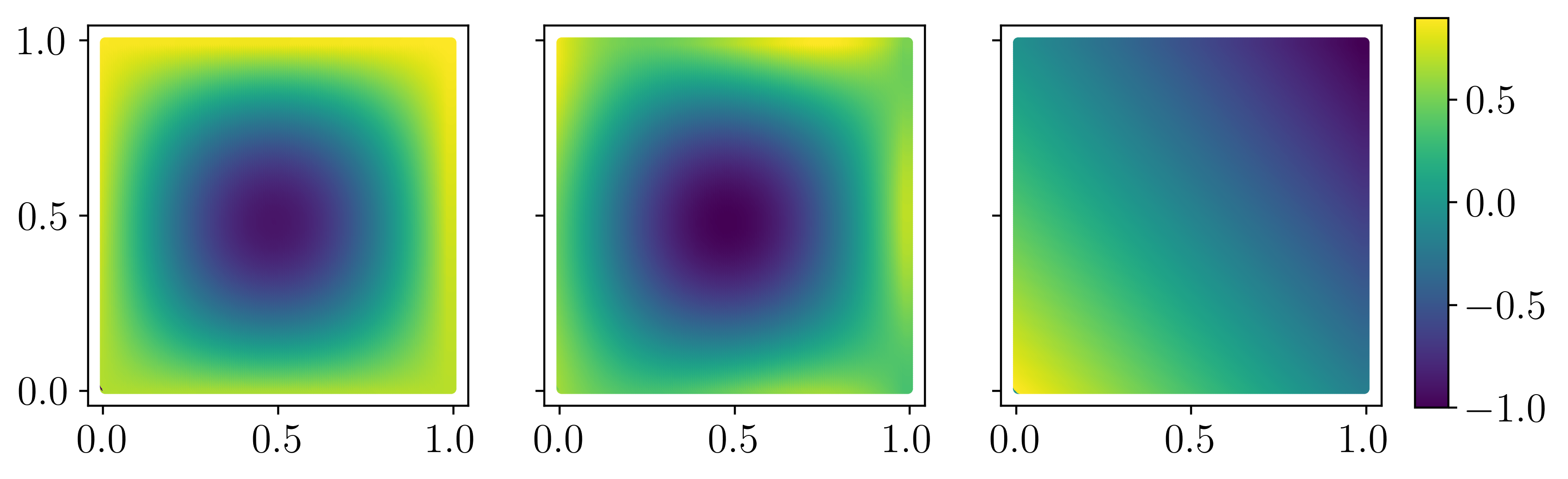}};
        \node[inner sep=0pt] (r1) at (-2.6,1.4)
    {\small 
        \footnotesize $u_\theta-u^\ast$
    };
        \node[inner sep=0pt] (r1) at (-0.35,1.4)
    {
    \footnotesize Energy NG  
    };
        \node[inner sep=0pt] (r1) at (2,1.4)
    {\footnotesize Vanilla gradient 
    };
    \end{tikzpicture}
    
    \caption{Shown are the error $u_\theta - u^*$ and the push forwards of the energy NG  and vanilla gradient; all functions  normed to lie in $[-1,1]$ to allow for a visual comparison.}\label{fig:poisson_pushs}
\end{figure}

\mzz
\subsection{An Example in Higher Dimensions}
As an example in higher dimensions we consider again the Poisson equation in five spatial dimensions
    \begin{align*}
        -\Delta u & =  f \quad\quad \ \ \ \quad \quad \quad \text{in } [0, 1]^{5}, \\
        u(x) & = \sum_{k=1}^5 \sin(\pi x_k) \quad \ \text{on } \partial[0,1]^{5}.
    \end{align*}

We use the manufactured solution
\[ 
    u^\ast\colon \mathbb R^{5} \to\mathbb R, \quad   x \mapsto \sum_{k=1}^5 \sin(\pi x_k) 
\]
hence $f = \pi^2 u^\ast$. For a given set of interior and boundary collocation points $(x_1^i, \dots ,x_5^i)_{i=1,\dots,N_\Omega}\subseteq [0, 1]^{5}$ and $(x^{b, i}_1,\dots,x^{b, i}_5)_{i=1,\dots,N_{\partial\Omega}}\subseteq \partial[0, 1]^{5}$ we define the loss function and energy inner product exactly as in equation \eqref{eq:poisson_loss} and \eqref{eq:poisson_energy_product}. 
In this example we demonstrate batched training by repeatedly drawing $N_\Omega = 3000$ random interior and $N_{\partial\Omega} = 500$ boundary collocation points in every iteration of the training process. We use a shallow neural network with input dimension 5 and 64 hidden neurons and hyperbolic tangent activation. 

As presented in Figure~\ref{fig:poisson5d_runtimes}, the energy natural gradient method 
produces highly accurate solutions in relatively short time. The convergence behavior of Adam and vanilla gradient descent again display a quickly saturating behavior and neither is able to produce competitively accurate solutions. In this example the quasi-Newton method BFGS was not able to produce more accurate solutions than the Adam optimizer. Again, E-NGD is not only the most accurate optimizer by two orders of magnitude but it achieves this accuracy while being given one order of magnitude less time.

\begin{figure}[]
    \centering
    \includegraphics[width=\linewidth]{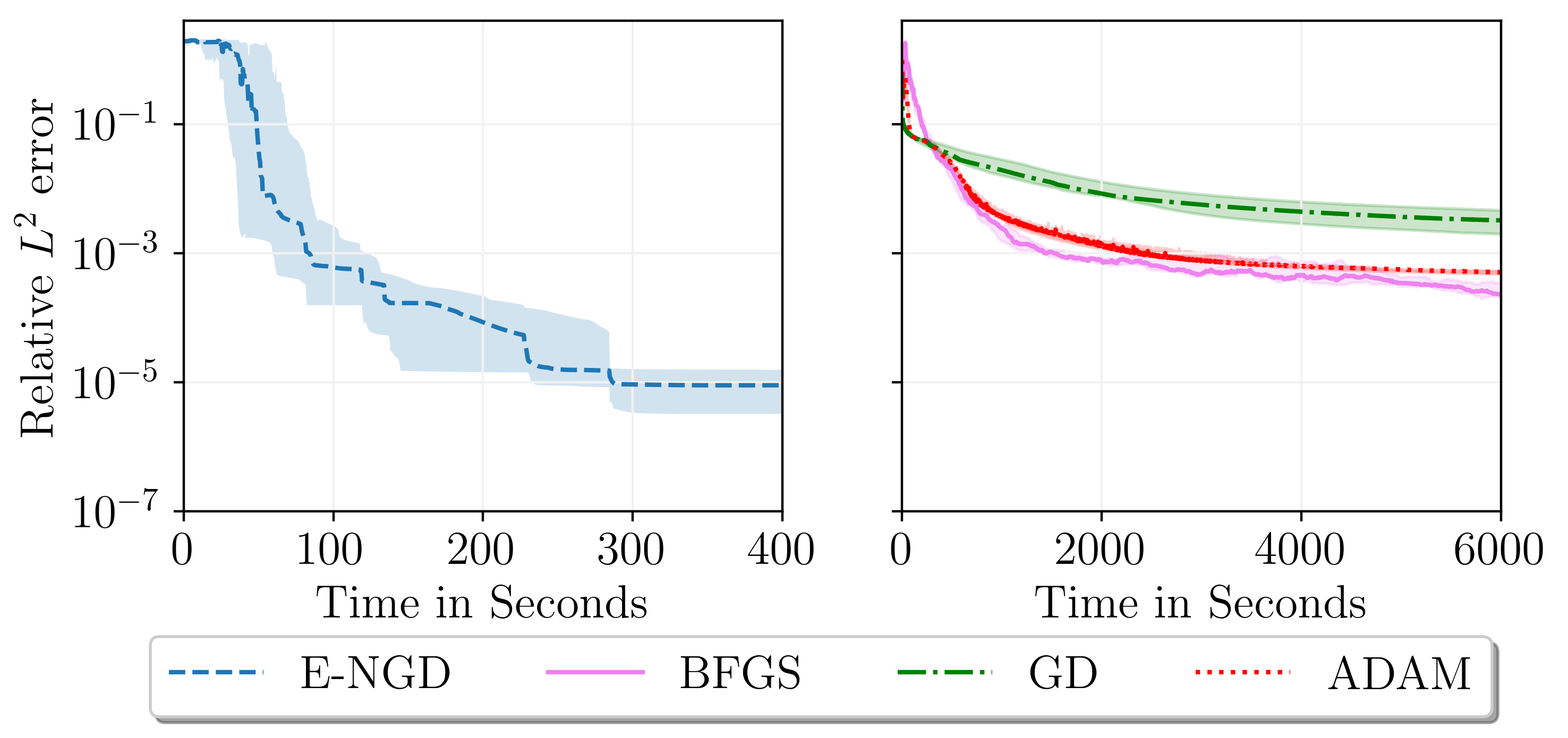}
    \vspace{-0.5cm}
    \caption{ 
    Median relative $L^2$ errors over computation time in seconds for the Poisson equation in five dimensions over 10
    initializations for the optimizers: energy NG descent, vanilla gradient descent and Adam; the shaded area denotes the region between the first and third quartile. Note the different scaling of the time axis for the two plots.
    }\label{fig:poisson5d_runtimes}
\end{figure}

\eee

\subsection{Heat Equation}
Let us consider the one-dimensional heat equation
\begin{align*}
    \partial_t u(t,x) &= \frac{1}{4}\partial_x^2u(t,x) \quad \text{for }(t,x)\in[0,1]^2
    \\
    u(0,x) &= \sin(\pi x) \qquad\;\, \text{for }x\in[0,1]
    \\
    u(t,x) &= 0 \qquad\qquad\quad\text{for }(t,x)\in[0,1]\times\{0,1\}.
\end{align*}
The solution is given by
\begin{equation*}
    u^*(t,x) = \exp\left(-\frac{\pi^2 t}{4}\right)\sin(\pi x)
\end{equation*}
and the PINN loss is
\begin{align*}
    L(\theta) &= \frac{1}{N_{\Omega_T}} \sum_{i=1}^{N_{\Omega_T}} \left( \partial_t u_\theta(t_i,x_i) - \frac14\partial_x^2 u_\theta(t_i, x_i) \right)^2 
    \\ 
    &\quad+ \frac{1}{N_\text{in}}\sum_{i=1}^{N_\Omega}\left(u_\theta(0,x_i^{\text{in}}) - \sin(\pi x_i^{\text{in}}) \right)^2
    \\&\quad +
    \frac{1}{N_{\partial\Omega}}\sum_{i=1}^{N_{\partial\Omega}}u_\theta(t^b_i,x^b_i)^2,
\end{align*}
where $\{ (t_i,x_i) \}_{i=1,\dots, N_{\Omega_T}}$ denote collocation points in the interior of the space-time cylinder, $\{ (t_i^b,x_i^b) \}_{i=1,\dots,N_{\partial\Omega}}$ denote collocation points on the spatial boundary and $\{ (x_i^{\text{in}}) \}_{i=1,\dots,N_{\text{in}}}$ denote collocation points for the initial condition. The energy inner product is defined on the space
\begin{equation*}
    a\colon\left(H^1(I,L^2(\Omega)) \cap L^2(I,H^2(\Omega))\right)^2 \to \mathbb R 
\end{equation*}
and given by
\begin{align*}
    a(u,v) &= \int_0^1\int_{\Omega}\left(\partial_t u - \frac14 \partial_x^2 u\right)\left(\partial_t v - \frac14 \partial_x^2 v\right)\,\mathrm dx\mathrm dt
    \\
    &\quad +
    \int_\Omega u(0,x)v(0,x)\, \mathrm dx + \int_{I\times\partial\Omega}uv \,\mathrm ds \mathrm dt.
\end{align*}
In our implementation, the inner product is discretized by the same quadrature points as in the definition of the loss function. 
The network architecture and the training process are identical to the previous example of the Poisson problem and we run the two NG methods for $2\cdot 10^3$ iterations.

\mzz
Also in this example, the energy natural gradient approach shows its high accuracy and efficiency. 
We refer to Table~\ref{table:heat} for the relative $L^2$ errors after training, Figure~\ref{fig:heat} for a visualization of the training process,  Table~\ref{table:heat_runtimes} for run-times and Figure \ref{fig:heat_pushs} for a visual comparison of the different gradients. 
\begin{table}[h]
\begin{center}
\begin{tabular}{|c|c|c|c|}
\hline

                & Median & Minimum & Maximum  \\
\hline 
GD  
& $1.6\cdot10^{-2}$  & $5.0\cdot10^{-3}$ & $4.2\cdot10^{-2}$   \\
Adam & $1.0\cdot10^{-3}$ & $6.4\cdot10^{-4}$  & $1.4\cdot10^{-3}$  \\
H-NGD  
& $4\cdot10^{-1}$  & $3\cdot10^{-1}$ & $5\cdot10^{-1}$  \\
E-NGD  
& $\mathbf{6.3\cdot 10^{-6}}$ & $\mathbf{2.3\cdot 10^{-6}}$  & $5.6\cdot10^{-1}$  \\
BFGS  
& $1.4\cdot 10^{-4}$ & $7.6\cdot10^{-5}$  & $\mathbf{3.3\cdot10^{-4}}$  \\
\hline
\end{tabular}
\caption{Median, minimum and maximum of the relative $L^2$ errors for the heat equation achieved by different optimizers over $10$ different initializations. Here, H-NGD and E-NGD is run for $2\cdot 10^3$ and the other methods for $2\cdot10^5$ iterations.}\label{table:heat}
\end{center}
\end{table}
\begin{figure}[h]
    \centering
    \includegraphics[width=\linewidth]{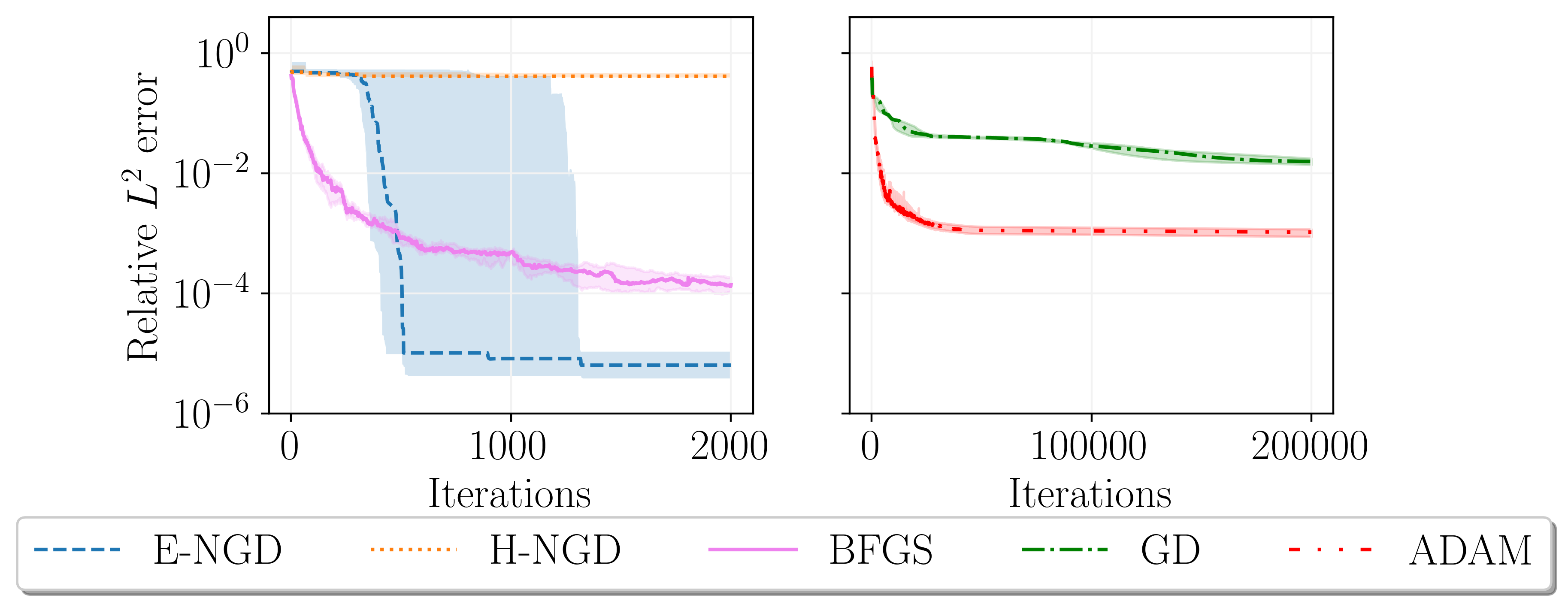}
    \vspace{-.5cm}
    \caption{\mzz Relative $L^2$ errors for the heat equation example throughout the training process for the five optimizers
    . The shaded area displays the region between the first and third quartile of 10 runs for different initializations. Note that GD and Adam are run for 100 times more iterations and GD, Adam and BFGS are given significantly more computation time than NGD, see~Table~\ref{table:heat_runtimes}. \eee
    }\label{fig:heat}
\end{figure}
\begin{table}[h]
\begin{center}
\begin{tabular}{|c|c|c|}
\hline
       & Time Iteration      & Time Full Optimization   \\
\hline 
GD     & $\mathbf{2.2\cdot 10^{-2}} \textbf{s}$  & $1 \textrm{h } 12 \textrm{min}$     \\
Adam   & $3.8\cdot 10^{-2} \textrm{s}$  & $ 2\textrm{h } 6 \textrm{min} $      \\
E-NGD  & $8.3\cdot 10^{-2} \textrm{s}$  & $\mathbf{2} \textbf{min } \mathbf{48} \textbf{s}$      \\ 
BFGS   & $4.3 \textrm{s}$               & $35 \textrm{min } 48\textrm{s}$      \\ 
\hline
\end{tabular}
\caption{\mzz Computational time for the optimizers  
for the heat equation. For the time per iteration we averaged over 100 iterations. The full optimization time is calculated as the product of total iteration count and averaged iteration time. The experiments were conducted on a single NVIDIA RTX 3080 Laptop GPU. \eee}\label{table:heat_runtimes}
\end{center}
\end{table}
\begin{figure}[h]
    \centering
    \begin{tikzpicture}
        \node[inner sep=0pt] (r1) at (0,0)
    {\includegraphics[width=\linewidth]{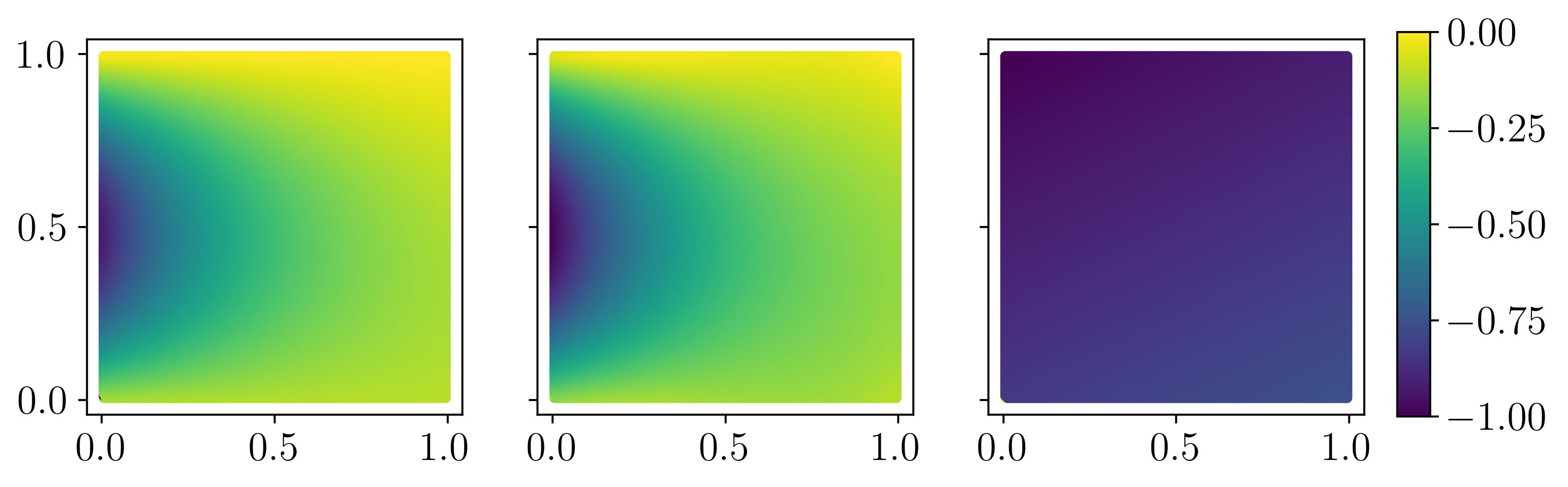}};
        \node[inner sep=0pt] (r1) at (-2.6,1.4)
    {\small   
        \footnotesize $u_\theta-u^\ast$
    };
        \node[inner sep=0pt] (r1) at (-0.35,1.4)
    {
    \footnotesize Energy NG 
    };
        \node[inner sep=0pt] (r1) at (2,1.4)
    {\footnotesize Vanilla gradient
    };
    \end{tikzpicture}
    
    \caption{The first image shows $u_\theta - u^*$, the second image is the computed natural gradient and the last image is the pushforward of the standard parameter gradient. All gradients are pointwise normed to $[-1,1]$ to allow visual comparison.}\label{fig:heat_pushs}
\end{figure}
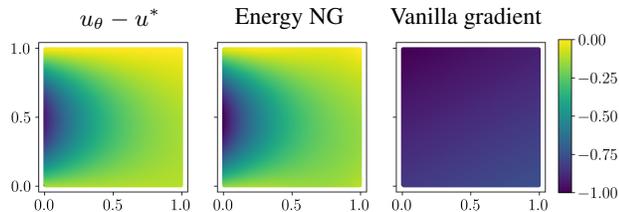
Note again the saturation of the conventional optimizers above $10^{-3}$ relative $L^2$ error although they are given one order of magnitude more computation time. 
\eee
Similar to the Poisson equation, the Hilbert NG descent is not an efficient optimization algorithm for the problem at hand which stresses again the importance of the Hessian information. 

\mzz
Note also, that while E-NG is highly efficient for most initializations we observed that sometimes a failure to train can occur, compare to Table~\ref{table:heat}. We did also note that conducting a pre-training, for instance with GD or Adam was in most cases able to circumvent this issue.
\eee

\subsection{A Nonlinear Example with the Deep Ritz Method}\label{subsec:nonlinearDR}

We test the energy natural gradient method for a nonlinear problem utilizing the deep Ritz formulation. Consider the one-dimensional variational problem of finding the minimizer of the energy
\begin{equation}
    E(u) \coloneqq 
    \frac12\int_\Omega |u'|^2\mathrm dx + \frac14 \int_\Omega u^4\mathrm dx - \int_\Omega fu\,\mathrm dx
\end{equation}
with $\Omega = [-1,1]$, $f(x) = \pi^2\cos(\pi x)+\cos^3(\pi x)$. 
The associated Euler Lagrange equations yield the nonlinear PDE
\begin{align}
    \begin{split}\label{eq:nonlinear}
        -u'' + u^3 &= f \quad \text{in }\Omega
    \\ 
    \partial_n u &=0 \quad \text{on }\partial\Omega
    \end{split}
\end{align}
and hence 
the minimizer is given by  
$u^\ast(x) = \cos(\pi x)$
. 
Since the energy is not quadratic, the energy inner product depends on $u\in H^1(\Omega)$ and is given by
\begin{equation*}
    D^2E(u)(v,w) = \int_\Omega v'w'\,\mathrm dx + 3\int_\Omega u^2vw\,\mathrm dx.
\end{equation*}
To discretize the energy and the inner product we use trapezoidal integration with $2\cdot10^4$ equi-spaced quadrature points. 
We use a shallow neural network 
of width of 32 neurons and a hyperbolic tangent as an activation function.

Once more, we observe that the energy NG updates efficiently lead to a very accurate approximation of the solution, see Figure~\ref{fig:nonlinear} for a visualization of the training process and Table~\ref{table:nonlinear} for the obtained relative $L^2$ errors. The computation times for the individual methods are reported in Table~\ref{table:nonlinear_runtimes}. 
\begin{table}[h]
\begin{center}
\begin{tabular}{|c|c|c|c|}
\hline


                & Median & Minimum & Maximum  \\
\hline 
GD 
& $2.2\cdot10^{-4}$  & $1.2\cdot10^{-4}$ & $2.6\cdot10^{-4}$   \\
Adam 
& $5.3\cdot10^{-5}$   & $2.4\cdot10^{-5}$ & $1.1\cdot10^{-4}$  \\
H-NGD 
& $\mathbf{1.0\cdot10^{-8}}$  & $6.3\cdot10^{-9}$  & $1.0$  \\
E-NGD 
& $1.3\cdot10^{-8}$  & $\mathbf{6.0\cdot10^{-9}}$  & $1.0$  \\
BFGS 
& $1.2\cdot10^{-5}$  & $4.7\cdot10^{-6}$  & $\mathbf{2.2\cdot10^{-5}}$  \\
\hline
\end{tabular}
\caption{ 
Median, minimum and maximum of the relative $L^2$ errors for the nonlinear problem achieved by different optimizers over $10$ different initializations. Here, H-NGD, E-NGD and BFGS is run for $500$ and the other methods for $2\cdot10^5$ iterations.
}\label{table:nonlinear}
\end{center}
\end{table}
\begin{figure}[h]
    \centering
    \includegraphics[width=\linewidth]{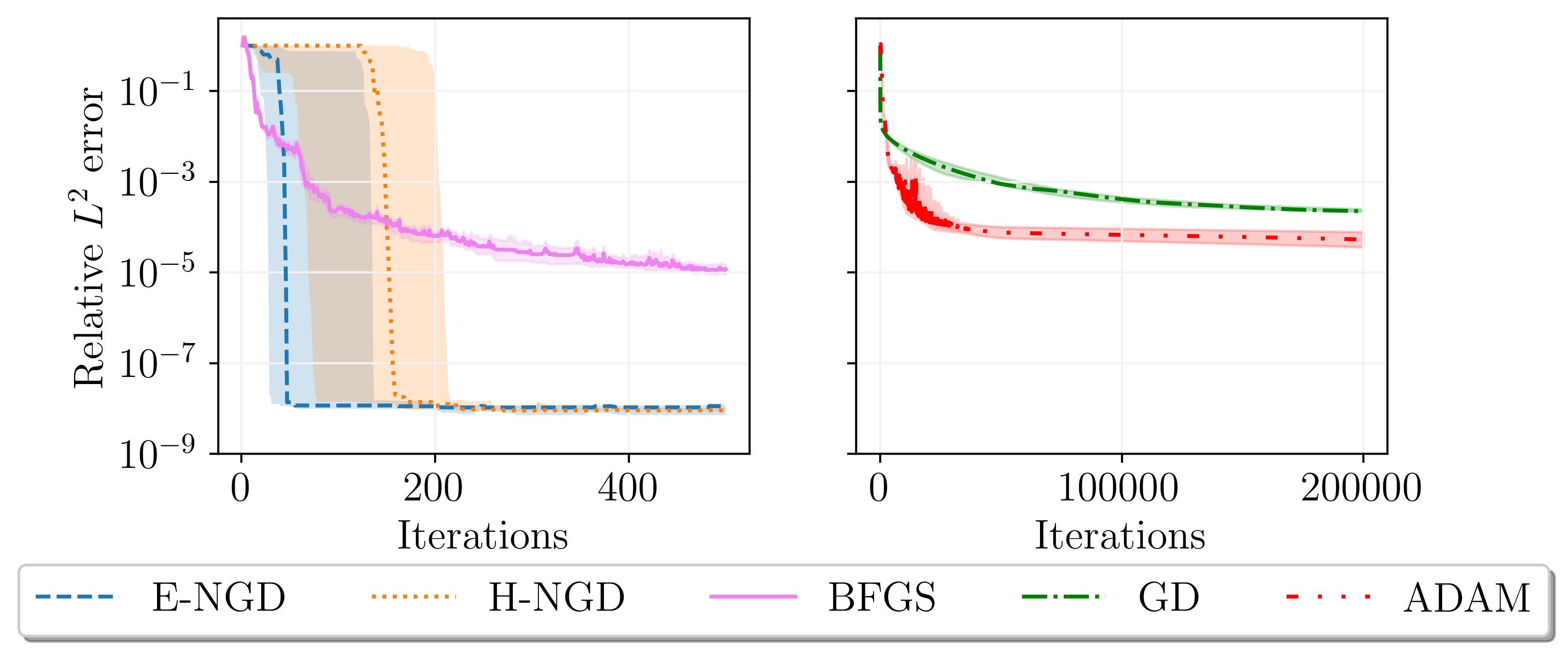}
    \vspace{-.5cm}
    \caption{\mzz
    Relative $L^2$ errors for the nonlinear example throughout the training process for the five optimizers
    . The shaded area displays the region between the first and third quartile of 10 runs for different initializations
    . Note that GD and Adam are run for $400$ times more iterations and GD, Adam and BFGS are given significantly more computation time than NGD, see~Table~\ref{table:deep_runtimes}.
    }\label{fig:nonlinear}
\end{figure}
\begin{table}[h]
\begin{center}
\begin{tabular}{|c|c|c|}
\hline
       & Time per Iteration      & Full Optimization Time   \\
\hline 
GD     & $\mathbf{3.7\cdot 10^{-2}} \textbf{ s}$  & $2 \textrm{h } 3 \textrm{min}$     \\
Adam   & $5.8\cdot 10^{-2} \textrm{s}$  & $3 \text{h } 13 \textrm{min}$      \\
H-NGD  & $8.3\cdot 10^{-2} \textrm{s}$  & $\mathbf{42} \textbf{s}$      \\ 
E-NGD  & $8.6\cdot 10^{-2} \textrm{s}$  & $43 \textrm{s}$      \\ 
BFGS   & $7.2 \textrm{s}$               & $1 \text{h }$      \\ 
\hline
\end{tabular}
\caption{\mzz Computational time for the optimizers 
for the nonlinear example. For the time per iteration we averaged over 100 iterations. The full optimization time is calculated as the product of total iteration count and averaged iteration time. The experiments were conducted on a single NVIDIA RTX 3080 Laptop GPU. \eee}\label{table:nonlinear_runtimes}
\end{center}
\end{table}
In this example, the Hilbert  NG descent is similarly efficient.  
Note that the energy inner product and the Hilbert space inner product  are very similar in this case. 
Again, Adam and standard gradient descent saturate early with much higher errors than the natural gradient methods.
We observe that obtaining high accuracy with the Deep Ritz method requires a highly accurate integration, which is why we used a comparatively fine grid and trapezoidal integration.

\section{Conclusion%
}
We propose to train physics informed neural networks with energy natural gradients, which correspond to the well-known concept of natural gradients combined with the geometric information of the Hessian in function space.
We show that the energy natural gradient update direction
corresponds to the Newton direction 
in function space, modulo an orthogonal projection onto the tangent space of the model. 
We demonstrate experimentally that this optimization achieves highly accurate PINN solutions, well beyond the the accuracy that can be obtained with standard optimizers \mzz
even if these methods are allowed several order of magnitude more computation time.
The proposed method is compatible with arbitrary discretizations of the integrals appearing in the objective and the gram matrix as with arbitrary network architectures.
Important future directions include the development of efficient implementations of energy natural gradients for large scale problems and the development of specialized initialization schemes.
\eee

\section*{Acknowledgements}
JM was supported by the Evangelisches Studienwerk e.V. (Villigst), the International Max Planck Research School for Mathematics in the Sciences (IMPRS MiS) and the European Research Council (ERC) under the EuropeanUnion’s Horizon 2020 research and innovation programme (grant number 757983). MZ acknowledges support from the Research Council of Norway (grant number 303362). 

\bibliography{bib}

\begin{thebibliography}{53}
\providecommand{\natexlab}[1]{#1}
\providecommand{\url}[1]{\texttt{#1}}
\expandafter\ifx\csname urlstyle\endcsname\relax
  \providecommand{\doi}[1]{doi: #1}\else
  \providecommand{\doi}{doi: \begingroup \urlstyle{rm}\Url}\fi

\bibitem[Amari(1998)]{amari1998natural}
Amari, S.
\newblock Natural gradient works efficiently in learning.
\newblock \emph{Neural computation}, 10\penalty0 (2):\penalty0 251--276, 1998.

\bibitem[Amari(2016)]{amari2016information}
Amari, S.
\newblock \emph{Information geometry and its applications}, volume 194.
\newblock Springer, Japan, 2016.

\bibitem[Amari \& Cichocki(2010)Amari and Cichocki]{amari2010information}
Amari, S. and Cichocki, A.
\newblock Information geometry of divergence functions.
\newblock \emph{Bulletin of the {P}olish academy of sciences. Technical
  sciences}, 58\penalty0 (1):\penalty0 183--195, 2010.

\bibitem[Bagnell \& Schneider(2003)Bagnell and Schneider]{bagnell2003covariant}
Bagnell, J.~A. and Schneider, J.~G.
\newblock Covariant policy search.
\newblock In \emph{IJCAI}, pp.\  1019--1024, 2003.

\bibitem[Beck et~al.(2020)Beck, Hutzenthaler, Jentzen, and
  Kuckuck]{beck2020overview}
Beck, C., Hutzenthaler, M., Jentzen, A., and Kuckuck, B.
\newblock An overview on deep learning-based approximation methods for partial
  differential equations.
\newblock \emph{arXiv preprint arXiv:2012.12348}, 2020.

\bibitem[Bradbury et~al.(2018)Bradbury, Frostig, Hawkins, Johnson, Leary,
  Maclaurin, Necula, Paszke, Vander{P}las, Wanderman-{M}ilne, and
  Zhang]{jax2018github}
Bradbury, J., Frostig, R., Hawkins, P., Johnson, M.~J., Leary, C., Maclaurin,
  D., Necula, G., Paszke, A., Vander{P}las, J., Wanderman-{M}ilne, S., and
  Zhang, Q.
\newblock {JAX}: composable transformations of {P}ython+{N}um{P}y programs,
  2018.
\newblock URL \url{http://github.com/google/jax}.

\bibitem[Cai et~al.(2019)Cai, Gao, Hou, Chen, Wang, He, Zhang, and
  Wang]{cai2019gram}
Cai, T., Gao, R., Hou, J., Chen, S., Wang, D., He, D., Zhang, Z., and Wang, L.
\newblock Gram-gauss-newton method: Learning overparameterized neural networks
  for regression problems.
\newblock \emph{arXiv preprint arXiv:1905.11675}, 2019.

\bibitem[Courte \& Zeinhofer(2023)Courte and Zeinhofer]{courte2023robin}
Courte, L. and Zeinhofer, M.
\newblock Robin {Pre-Training for the Deep Ritz Method}.
\newblock \emph{Northern Lights Deep Learning Conference}, 2023.

\bibitem[Davi \& Braga-Neto(2022)Davi and Braga-Neto]{davi2022pso}
Davi, C. and Braga-Neto, U.
\newblock Pso-pinn: Physics-informed neural networks trained with particle
  swarm optimization.
\newblock \emph{arXiv preprint arXiv:2202.01943}, 2022.

\bibitem[Daw et~al.(2022)Daw, Bu, Wang, Perdikaris, and
  Karpatne]{daw2022rethinking}
Daw, A., Bu, J., Wang, S., Perdikaris, P., and Karpatne, A.
\newblock Rethinking the importance of sampling in physics-informed neural
  networks.
\newblock \emph{arXiv preprint arXiv:2207.02338}, 2022.

\bibitem[Dissanayake \& Phan-Thien(1994)Dissanayake and
  Phan-Thien]{dissanayake1994neural}
Dissanayake, M. and Phan-Thien, N.
\newblock Neural-network-based approximations for solving partial differential
  equations.
\newblock \emph{communications in Numerical Methods in Engineering},
  10\penalty0 (3):\penalty0 195--201, 1994.

\bibitem[E \& Yu(2018)E and Yu]{weinan2018deep}
E, W. and Yu, B.
\newblock {The Deep Ritz Method: A Deep Learning-Based Numerical Algorithm for
  Solving Variational Problems}.
\newblock \emph{Communications in Mathematics and Statistics}, 6\penalty0
  (1):\penalty0 1--12, 2018.

\bibitem[E et~al.(2017)E, Han, and Jentzen]{weinan2017deep}
E, W., Han, J., and Jentzen, A.
\newblock Deep learning-based numerical methods for high-dimensional parabolic
  partial differential equations and backward stochastic differential
  equations.
\newblock \emph{Communications in Mathematics and Statistics}, 5\penalty0
  (4):\penalty0 349--380, 2017.

\bibitem[Gargiani et~al.(2020)Gargiani, Zanelli, Diehl, and
  Hutter]{gargiani2020promise}
Gargiani, M., Zanelli, A., Diehl, M., and Hutter, F.
\newblock On the promise of the stochastic generalized gauss-newton method for
  training dnns.
\newblock \emph{arXiv preprint arXiv:2006.02409}, 2020.

\bibitem[Han et~al.(2018)Han, Jentzen, and Weinan]{han2018solving}
Han, J., Jentzen, A., and Weinan, E.
\newblock Solving high-dimensional partial differential equations using deep
  learning.
\newblock \emph{Proceedings of the National Academy of Sciences}, 115\penalty0
  (34):\penalty0 8505--8510, 2018.

\bibitem[Hao et~al.(2021)Hao, Jin, Siegel, and Xu]{hao2021efficient}
Hao, W., Jin, X., Siegel, J.~W., and Xu, J.
\newblock An efficient greedy training algorithm for neural networks and
  applications in {PDEs}.
\newblock \emph{arXiv preprint arXiv:2107.04466}, 2021.

\bibitem[Kakade(2001)]{kakade2001natural}
Kakade, S.~M.
\newblock A natural policy gradient.
\newblock \emph{Advances in Neural Information Processing Systems}, 14, 2001.

\bibitem[Kovachki et~al.(2021)Kovachki, Li, Liu, Azizzadenesheli, Bhattacharya,
  Stuart, and Anandkumar]{kovachki2021neural}
Kovachki, N., Li, Z., Liu, B., Azizzadenesheli, K., Bhattacharya, K., Stuart,
  A., and Anandkumar, A.
\newblock Neural operator: Learning maps between function spaces.
\newblock \emph{arXiv preprint arXiv:2108.08481}, 2021.

\bibitem[Krishnapriyan et~al.(2021)Krishnapriyan, Gholami, Zhe, Kirby, and
  Mahoney]{krishnapriyan2021characterizing}
Krishnapriyan, A., Gholami, A., Zhe, S., Kirby, R., and Mahoney, M.~W.
\newblock Characterizing possible failure modes in physics-informed neural
  networks.
\newblock \emph{Advances in Neural Information Processing Systems},
  34:\penalty0 26548--26560, 2021.

\bibitem[Lagaris et~al.(1998)Lagaris, Likas, and
  Fotiadis]{lagaris1998artificial}
Lagaris, I.~E., Likas, A., and Fotiadis, D.~I.
\newblock Artificial neural networks for solving ordinary and partial
  differential equations.
\newblock \emph{IEEE transactions on neural networks}, 9\penalty0 (5):\penalty0
  987--1000, 1998.

\bibitem[Li \& Mont{\'u}far(2018)Li and Mont{\'u}far]{li2018natural}
Li, W. and Mont{\'u}far, G.
\newblock Natural gradient via optimal transport.
\newblock \emph{Information Geometry}, 1\penalty0 (2):\penalty0 181--214, 2018.
\newblock URL \url{https://doi.org/10.1007/s41884-018-0015-3}.

\bibitem[Li et~al.(2021)Li, Kovachki, Azizzadenesheli, liu, Bhattacharya,
  Stuart, and Anandkumar]{li2021fourier}
Li, Z., Kovachki, N.~B., Azizzadenesheli, K., liu, B., Bhattacharya, K.,
  Stuart, A., and Anandkumar, A.
\newblock Fourier neural operator for parametric partial differential
  equations.
\newblock In \emph{International Conference on Learning Representations}, 2021.
\newblock URL \url{https://openreview.net/forum?id=c8P9NQVtmnO}.

\bibitem[Lin et~al.(2021)Lin, Li, Osher, and Mont{\'u}far]{lin2021wasserstein}
Lin, A.~T., Li, W., Osher, S., and Mont{\'u}far, G.
\newblock Wasserstein proximal of gans.
\newblock In \emph{International Conference on Geometric Science of
  Information}, pp.\  524--533. Springer, 2021.

\bibitem[Lu et~al.(2021)Lu, Meng, Mao, and Karniadakis]{lu2021deepxde}
Lu, L., Meng, X., Mao, Z., and Karniadakis, G.~E.
\newblock Deepxde: A deep learning library for solving differential equations.
\newblock \emph{SIAM Review}, 63\penalty0 (1):\penalty0 208--228, 2021.

\bibitem[Martens(2020)]{martens2020new}
Martens, J.
\newblock New insights and perspectives on the natural gradient method.
\newblock \emph{The Journal of Machine Learning Research}, 21\penalty0
  (1):\penalty0 5776--5851, 2020.

\bibitem[Morimura et~al.(2008)Morimura, Uchibe, Yoshimoto, and
  Doya]{morimura2008new}
Morimura, T., Uchibe, E., Yoshimoto, J., and Doya, K.
\newblock A new natural policy gradient by stationary distribution metric.
\newblock In \emph{Joint European Conference on Machine Learning and Knowledge
  Discovery in Databases}, pp.\  82--97. Springer, 2008.

\bibitem[M\"uller \& Mont\'ufar(2022)M\"uller and
  Mont\'ufar]{Mueller2022Convergence}
M\"uller, J. and Mont\'ufar, G.
\newblock Geometry and convergence of natural policy gradients.
\newblock \emph{MPI MiS Preprint 31/2022}, 2022.
\newblock URL
  \url{https://www.mis.mpg.de/publications/preprints/2022/prepr2022-31.html}.

\bibitem[M{\"u}ller \& Zeinhofer(2022{\natexlab{a}})M{\"u}ller and
  Zeinhofer]{muller2022error}
M{\"u}ller, J. and Zeinhofer, M.
\newblock Error estimates for the deep ritz method with boundary penalty.
\newblock In \emph{Mathematical and Scientific Machine Learning}, pp.\
  215--230. PMLR, 2022{\natexlab{a}}.

\bibitem[M{\"u}ller \& Zeinhofer(2022{\natexlab{b}})M{\"u}ller and
  Zeinhofer]{muller2022notes}
M{\"u}ller, J. and Zeinhofer, M.
\newblock Notes on exact boundary values in residual minimisation.
\newblock In \emph{Mathematical and Scientific Machine Learning}, pp.\
  231--240. PMLR, 2022{\natexlab{b}}.

\bibitem[Nabian et~al.(2021)Nabian, Gladstone, and
  Meidani]{nabian2021efficient}
Nabian, M.~A., Gladstone, R.~J., and Meidani, H.
\newblock Efficient training of physics-informed neural networks via importance
  sampling.
\newblock \emph{Computer-Aided Civil and Infrastructure Engineering},
  36\penalty0 (8):\penalty0 962--977, 2021.

\bibitem[Nocedal \& Wright(1999)Nocedal and Wright]{nocedal1999numerical}
Nocedal, J. and Wright, S.~J.
\newblock \emph{Numerical optimization}.
\newblock Springer, 1999.

\bibitem[Nurbekyan et~al.(2022)Nurbekyan, Lei, and
  Yang]{nurbekyan2022efficient}
Nurbekyan, L., Lei, W., and Yang, Y.
\newblock Efficient natural gradient descent methods for large-scale
  optimization problems.
\newblock \emph{arXiv:2202.06236}, 2022.

\bibitem[Pascanu \& Bengio(2014)Pascanu and Bengio]{pascanu2014revisiting}
Pascanu, R. and Bengio, Y.
\newblock Revisiting natural gradient for deep networks.
\newblock In \emph{International Conference on Learning Representations}, 2014.
\newblock URL \url{https://openreview.net/forum?id=vz8AumxkAfz5U}.

\bibitem[Peters et~al.(2003)Peters, Vijayakumar, and
  Schaal]{peters2003reinforcement}
Peters, J., Vijayakumar, S., and Schaal, S.
\newblock Reinforcement learning for humanoid robotics.
\newblock In \emph{Proceedings of the third IEEE-RAS international conference
  on humanoid robots}, pp.\  1--20, 2003.

\bibitem[Raissi et~al.(2019)Raissi, Perdikaris, and
  Karniadakis]{raissi2019physics}
Raissi, M., Perdikaris, P., and Karniadakis, G.~E.
\newblock Physics-informed neural networks: A deep learning framework for
  solving forward and inverse problems involving nonlinear partial differential
  equations.
\newblock \emph{Journal of Computational physics}, 378:\penalty0 686--707,
  2019.

\bibitem[Ren \& Goldfarb(2019)Ren and Goldfarb]{ren2019efficient}
Ren, Y. and Goldfarb, D.
\newblock Efficient subsampled gauss-newton and natural gradient methods for
  training neural networks.
\newblock \emph{arXiv preprint arXiv:1906.02353}, 2019.

\bibitem[Ritz(1909)]{ritz1909neue}
Ritz, W.
\newblock {{\"U}ber eine neue Methode zur L{\"o}sung gewisser
  Variationsprobleme der mathematischen Physik.}
\newblock \emph{Journal f{\"u}r die reine und angewandte Mathematik (Crelles
  Journal)}, 1909\penalty0 (135):\penalty0 1--61, 1909.

\bibitem[Schraudolph(2002)]{schraudolph2002fast}
Schraudolph, N.~N.
\newblock Fast curvature matrix-vector products for second-order gradient
  descent.
\newblock \emph{Neural computation}, 14\penalty0 (7):\penalty0 1723--1738,
  2002.

\bibitem[Schwedes et~al.(2016)Schwedes, Funke, and Ham]{schwedes2016iteration}
Schwedes, T., Funke, S.~W., and Ham, D.~A.
\newblock An iteration count estimate for a mesh-dependent steepest descent
  method based on finite elements and {R}iesz inner product representation.
\newblock \emph{arXiv preprint arXiv:1606.08069}, 2016.

\bibitem[Schwedes et~al.(2017)Schwedes, Ham, Funke, and
  Piggott]{schwedes2017mesh}
Schwedes, T., Ham, D.~A., Funke, S.~W., and Piggott, M.~D.
\newblock Mesh dependence in {PDE}-constrained optimisation.
\newblock In \emph{Mesh Dependence in {PDE}-Constrained Optimisation}, pp.\
  53--78. Springer, 2017.

\bibitem[Shen et~al.(2020)Shen, Wang, Ribeiro, and Hassani]{shen2020sinkhorn}
Shen, Z., Wang, Z., Ribeiro, A., and Hassani, H.
\newblock Sinkhorn natural gradient for generative models.
\newblock \emph{Advances in Neural Information Processing Systems},
  33:\penalty0 1646--1656, 2020.

\bibitem[Sirignano \& Spiliopoulos(2018)Sirignano and
  Spiliopoulos]{sirignano2018dgm}
Sirignano, J. and Spiliopoulos, K.
\newblock {DGM: A deep learning algorithm for solving partial differential
  equations}.
\newblock \emph{Journal of computational physics}, 375:\penalty0 1339--1364,
  2018.

\bibitem[Thomas et~al.(2016)Thomas, Silva, Dann, and
  Brunskill]{thomas2016energetic}
Thomas, P., Silva, B.~C., Dann, C., and Brunskill, E.
\newblock Energetic natural gradient descent.
\newblock In \emph{International Conference on Machine Learning}, pp.\
  2887--2895. PMLR, 2016.

\bibitem[van~der Meer et~al.(2022)van~der Meer, Oosterlee, and
  Borovykh]{van2022optimally}
van~der Meer, R., Oosterlee, C.~W., and Borovykh, A.
\newblock Optimally weighted loss functions for solving pdes with neural
  networks.
\newblock \emph{Journal of Computational and Applied Mathematics},
  405:\penalty0 113887, 2022.

\bibitem[van Oostrum et~al.(2022)van Oostrum, M{\"u}ller, and
  Ay]{van2022invariance}
van Oostrum, J., M{\"u}ller, J., and Ay, N.
\newblock Invariance properties of the natural gradient in overparametrised
  systems.
\newblock \emph{Information Geometry}, pp.\  1--17, 2022.

\bibitem[Wang \& Yan(2022)Wang and Yan]{wang2022hessian}
Wang, L. and Yan, M.
\newblock Hessian informed mirror descent.
\newblock \emph{Journal of Scientific Computing}, 92\penalty0 (3):\penalty0
  1--22, 2022.

\bibitem[Wang et~al.(2021)Wang, Teng, and Perdikaris]{wang2021understanding}
Wang, S., Teng, Y., and Perdikaris, P.
\newblock Understanding and mitigating gradient flow pathologies in
  physics-informed neural networks.
\newblock \emph{SIAM Journal on Scientific Computing}, 43\penalty0
  (5):\penalty0 A3055--A3081, 2021.

\bibitem[Wang et~al.(2022{\natexlab{a}})Wang, Sankaran, and
  Perdikaris]{wang2022respecting}
Wang, S., Sankaran, S., and Perdikaris, P.
\newblock Respecting causality is all you need for training physics-informed
  neural networks.
\newblock \emph{arXiv preprint arXiv:2203.07404}, 2022{\natexlab{a}}.

\bibitem[Wang et~al.(2022{\natexlab{b}})Wang, Yu, and Perdikaris]{wang2022and}
Wang, S., Yu, X., and Perdikaris, P.
\newblock When and why {PINNs} fail to train: A neural tangent kernel
  perspective.
\newblock \emph{Journal of Computational Physics}, 449:\penalty0 110768,
  2022{\natexlab{b}}.

\bibitem[Weinan et~al.(2021)Weinan, Han, and Jentzen]{weinan2021algorithms}
Weinan, E., Han, J., and Jentzen, A.
\newblock Algorithms for solving high dimensional {PDEs}: from nonlinear monte
  carlo to machine learning.
\newblock \emph{Nonlinearity}, 35\penalty0 (1):\penalty0 278, 2021.

\bibitem[Wu et~al.(2023)Wu, Zhu, Tan, Kartha, and Lu]{wu2023comprehensive}
Wu, C., Zhu, M., Tan, Q., Kartha, Y., and Lu, L.
\newblock A comprehensive study of non-adaptive and residual-based adaptive
  sampling for physics-informed neural networks.
\newblock \emph{Computer Methods in Applied Mechanics and Engineering},
  403:\penalty0 115671, 2023.

\bibitem[Zapf et~al.(2022)Zapf, Haubner, Kuchta, Ringstad, Eide, and
  Mardal]{zapf2022investigating}
Zapf, B., Haubner, J., Kuchta, M., Ringstad, G., Eide, P.~K., and Mardal, K.-A.
\newblock Investigating molecular transport in the human brain from mri with
  physics-informed neural networks.
\newblock \emph{Scientific Reports}, 12\penalty0 (1):\penalty0 1--12, 2022.

\bibitem[Zeng et~al.(2022)Zeng, Bryngelson, and Schaefer]{zeng2022competitive}
Zeng, Q., Bryngelson, S.~H., and Schaefer, F.~T.
\newblock Competitive physics informed networks.
\newblock In \emph{ICLR 2022 Workshop on Gamification and Multiagent
  Solutions}, 2022.
\newblock URL \url{https://openreview.net/forum?id=rMz_scJ6lc}.

\end{thebibliography}
\bibliographystyle{icml2023}

\clearpage

\appendix

\section{Proofs Regarding NGs in Function Space}\label{sec:proofs}

We follow an analogue approach to~\cite{van2022invariance}, which considers finite dimensional spaces. 

\begin{lemma}\label{app:lem}
Let $X$ be a vector space with a scalar product $\langle\cdot, \cdot\rangle\colon X\times X\to\mathbb R$ and consider a linear map $A\colon\mathbb R^p\to X$ for some $p\in\mathbb N$. Let $G\in\mathbb R^{p\times p}$ be given by $G_{ij}\coloneqq \langle Ae_i, Ae_j\rangle$ and consider the adjoint operator $A^\ast\colon X\to\mathbb R^p$  given by 
\begin{equation}
    A^\ast y \coloneqq \sum_{i=1}^p \langle y, Ae_i \rangle e_i.
\end{equation}
Then it holds that
\begin{equation}
    AG^+A^\ast x = \Pi_{R(A)}(x),
\end{equation}
where $\Pi_{R(A)}(x)$ denotes the projection of $x$ onto the range $R(A)=\{Av:v\in\mathbb R^p\}$ of $A$, which is the unique element satisfying
\begin{equation}
    \langle \Pi_{R(A)}(x), z\rangle = \langle x, z\rangle \quad \text{for all } z\in R(A).
\end{equation}
\end{lemma}
\begin{proof}
It is elementary to check that the adjoint satisfies $\langle A^\ast x,v\rangle = \langle x, Av\rangle$. 
Picking some orthonormal basis $(b_i)_{i =1, \dots, d}$ of $R(A)$, the orthogonal projection of $x\in X$ to $R(A)$ exists and is given by $\sum_i \langle x, b_i\rangle b_i$. 
Without loss of generality we can assume $x\in R(A)$ and otherwise replace $x$ by its projection onto $R(A)$ since $A^\ast$ vanishes on $R(A)^\perp$.

Let us use the notation $v_i\coloneqq Ae_i$. 
Note that clearly $AG^+A^\ast x\in R(A)$. Hence, it remains to show that $\langle AG^+A^\ast x, v_i\rangle = \langle x, v_i\rangle$ for all $i=1, \dots, p$. It holds that $A^\ast v_i = \sum_j \langle Ae_i, Ae_j\rangle e_j = Ge_i$ and we can express $x = \sum_i a_i v_i$. Using the symmetry of $G$ we can compute
\begin{align}
    \begin{split}
            \langle AG^+A^\ast x, v_i\rangle & = \langle G^+A^\ast x, A^\ast v_i\rangle
            \\ & = \sum_{j} a_j \langle G^+ A^\ast v_j, Ge_i\rangle 
            \\ & = \sum_{j} a_j \langle G G^+ Ge_j, e_i\rangle 
            \\ & = \sum_{j} a_j \langle Ge_j, e_i\rangle 
            \\ & = \sum_{j} a_j \langle A^\ast v_j, e_i\rangle 
            \\ & = \sum_{j} a_j \langle v_j, A e_i\rangle 
            \\ & = \langle x, v_i\rangle,
    \end{split}
\end{align}
which completes the proof.
\end{proof}

\begin{theorem}[NG for Hilbert Manifolds]\label{app:thm:NG}
Let $(\mathcal M, g)$ be a Riemannian Hilbert manifold with model space $X$, where for any $x\in\mathcal M$ the Riemannian metric $g_x$ defines a scalar product on the tangent space $T_x\mathcal M\cong X$ rendering $T_x\mathcal M$ complete. Assume a differentiable objective function $E\colon\mathcal M\to\mathbb R$ and a differentiable parametrization $P\colon\mathbb R^p\to\mathcal M$ and define the Gram matrix in the usual way $G(\theta)_{ij}\coloneqq g_{P(\theta)}(\partial_{\theta_i}P(\theta), \partial_{\theta_j}P(\theta))$ \mzz and consider the objective function $L\colon\mathbb R^p\to\mathbb R, \theta\mapsto E(u_\theta)$.\eee
Then it holds that
\begin{align}
    DP_\theta G(\theta)^+\nabla L(\theta) = \Pi_{T_\theta P(\mathbb R^p)}\nabla E(P(\theta)).
\end{align}
\end{theorem}
\begin{proof}
This follows directly from Lemma~\ref{app:lem} by setting $X\coloneqq T_{P(\theta)}\mathcal M$ and $A=DP_\theta$, where by the gradient chain rule it holds that $\nabla L(\theta) = DP(\theta)^\ast \nabla E(u_\theta)$. 
\end{proof}

\ENGFunctionSpace*
\begin{proof}
The case of strongly convex energy $E$ is a falls into the setting of Theorem~\ref{app:thm:NG} by defining the Riemannian metric via $g_u\coloneqq DE^2(u)$. It remains to show that the Riemannian gradient with respect to the metric induced by the second derivative $D^2E$ is given by $D^2E(u)^{-1}\nabla E(u)$. This follows from 
\begin{align}
    \begin{split}
        D^2E(u)(D^2E(u)^{-1}\nabla E(u), v) = \langle \nabla E(u), v \rangle = DE(u)v.
    \end{split}
\end{align}

Consider now the case of a symmetric quadratic function $E$ with positiv definite second derivative $D^2E$ and assume that $E$ admits a unique minimizer $u^\ast\in X$. 
Lemma~\ref{app:lem} with $A = DP_\theta$ implies 
\begin{align}
    \begin{split}
        DP_\theta G(\theta)^+ DP_\theta^{\ast, a} (u-u^\ast) = \Pi_{T_\theta \mathcal F_\Theta} (u-u^\ast), 
    \end{split}
\end{align}
where $DP_\theta^{\ast, a}$ denotes the adjoint of $DP_\theta$ with respect to the inner product $a$. 
Hence, it remains to show $\nabla L(\theta) = DP_\theta^{\ast, a} (u-u^\ast)$. 
Note that $E(u) = \frac12a(u-u^\ast, u-u^\ast) + c$ for a suitable constant $c\in\mathbb R$. 
This follows from the computation
\begin{align}
    \begin{split}
        \langle DP_\theta^{\ast, a} (u-u^\ast), e_i \rangle_{\mathbb R^p} & = a(u_\theta - u^\ast, DP_\theta e_i) 
        \\& = a(u_\theta - u^\ast, \partial_{\theta_i} u_\theta)
        \\ & = DE(u_\theta)\partial_{\theta_i}u_\theta \\ & = \partial_{\theta_i}L(\theta),
    \end{split}
\end{align}
where we used the chain rule in the last step. 
\end{proof}

\section{Additional Resources for the Experiments}
\subsection{Relative $H^1$ Errors}
For completeness sake, we report the $H^1$ errors of the three experiments in the main part of the manuscript.

\paragraph{Relative $H^1$ errors after training}
First, we present the relative $H^1$ errors obtained at the end of training in Tables~\ref{table:poisson_h1}-\ref{table:nonlinear_h1}.

\begin{table}[h!]
\begin{center}
\begin{tabular}{|c|c|c|c|}
\hline
& Median & Minimum & Maximum  \\ \hline
GD & $9.3\cdot10^{-2}$  & $2.6\cdot10^{-2}$ & $1.4\cdot10^{-1}$   \\
Adam & $1.5\cdot10^{-2}$  & $9.1\cdot10^{-3}$  & $2.5\cdot10^{-2}$  \\
H-NGD & $7.0$  & $5.2$  & $12.4$  \\
E-NGD & $\mathbf{4.9\cdot10^{-6}}$  & $\mathbf{2.6\cdot10^{-6}}$ & $\mathbf{1.0\cdot10^{-5}}$  \\
BFGS & $6.0\cdot10^{-3}$  & $1.3\cdot10^{-3}$ & $1.7\cdot10^{-2}$  \\
\hline
\end{tabular}
\caption{\mzz Median, minimum and maximum of the relative $H^1$ errors for the Poisson equation achieved by different optimizers over $10$ different initializations. Here, H-NGD, E-NGD and BFGS is run for $500$ and the other methods for $2\cdot10^5$ iterations. \eee}
\label{table:poisson_h1}
\end{center}
\end{table}

\begin{table}[h!]
\begin{center}
\begin{tabular}{|c|c|c|c|}
\hline
& Median & Minimum & Maximum  \\ \hline
GD & $1.8\cdot10^{-1}$  & $7.4\cdot10^{-2}$ & $4.5\cdot10^{-1}$   \\
Adam & $1.7\cdot10^{-2}$  & $1.2\cdot10^{-2}$  & $2.5\cdot10^{-2}$  \\
H-NGD & $2.9$  & $2.6$  & $3.0$  \\
E-NGD & $\mathbf{1.8\cdot10^{-4}}$  & $\mathbf{8.2\cdot10^{-5}}$ & $3.1$  \\
BFGS & $2.8\cdot10^{-3}$  & $1.3\cdot10^{-3}$  & $\mathbf{6.7\cdot10^{-3}}$  \\
\hline
\end{tabular}
\caption{Median, minimum and maximum of the relative $H^1$ errors for the heat equation achieved by different optimizers over $10$ different initializations. Here, H-NGD, E-NGD and BFGS is run for $2000$ and the other methods for $2\cdot10^5$ iterations.}
\end{center}
\end{table}

\begin{table}[h!]
\begin{center}
\begin{tabular}{|c|c|c|c|}
\hline
& Median & Minimum & Maximum  \\ \hline
GD & $6.1\cdot10^{-3}$  & $3.6\cdot10^{-3}$ & $9.1\cdot10^{-3}$   \\
Adam & $1.8\cdot10^{-3}$  & $7.5\cdot10^{-4}$  & $3.7\cdot10^{-3}$  \\
H-NGD & $\mathbf{8.1\cdot10^{-7}}$  & $\mathbf{5.5\cdot10^{-7}}$  & $8.3$  \\
E-NGD & $9.3\cdot10^{-7}$  & $6.0\cdot10^{-7}$ & $8.3$  \\
BFGS & $2.4\cdot10^{-4}$  & $7.9\cdot10^{-5}$ & $\mathbf{4.1\cdot10^{-4}}$  \\
\hline
\end{tabular}
\caption{\mzz
Median, minimum and maximum of the relative $H^1$ errors for the nonlinear equation achieved by different optimizers over $10$ different initializations. Here, H-NGD, E-NGD and BFGS is run for $500$ and the other methods for $2\cdot10^5$ iterations. At the moment only a single initialization of BFGS is run. \eee}
\end{center}\label{table:nonlinear_h1}
\end{table}

\paragraph{Relative $H^1$ Errors During Training}
Furthermore, we provide also the visualizations of the training processes of the experiments of the main section when the error is measured in $H^1$ norm.

\begin{figure}
    \centering
    \includegraphics[width=\linewidth]{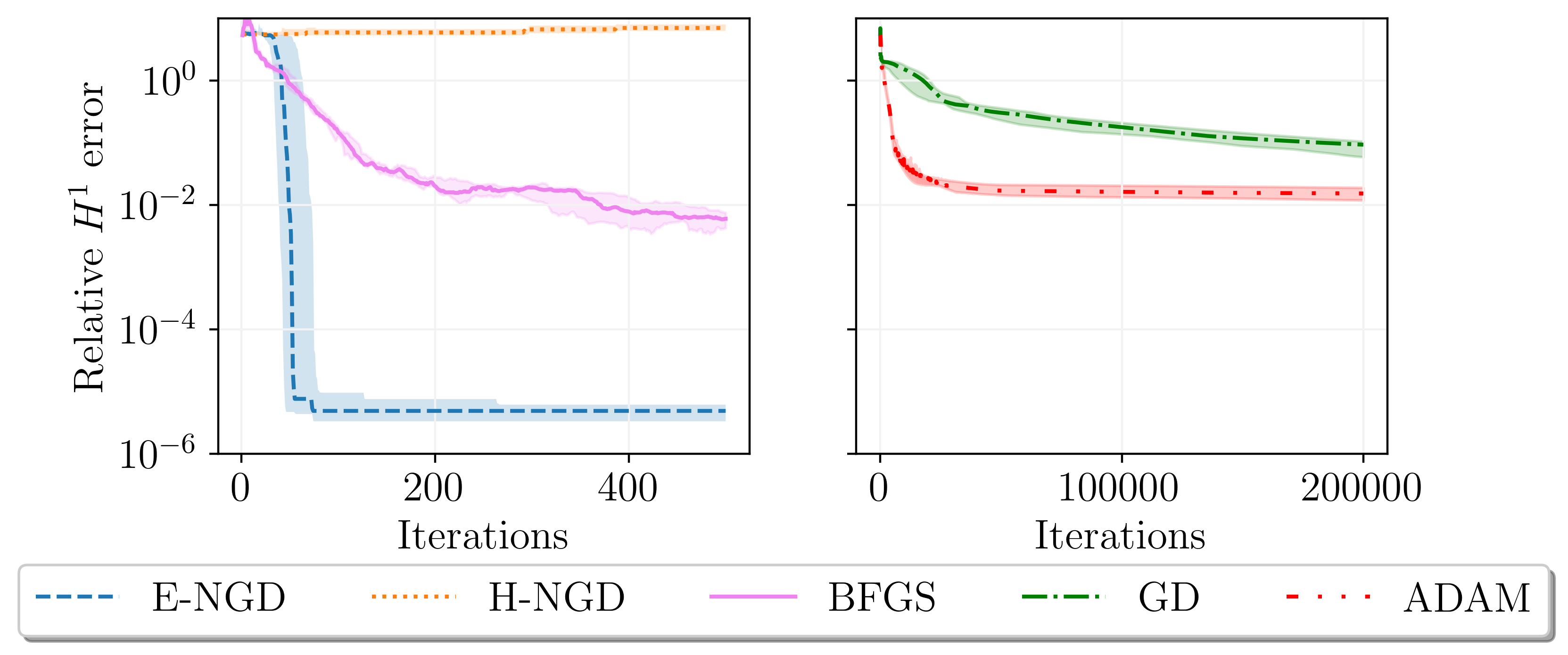}
    \vspace{-.5cm}
    \caption{\mzz The plot shows the median of the relative $H^1$ errors for the Poisson equation throughout the training process for the five optimizers: energy natural gradient descent, Hilbert natural gradient descent, BFGS, vanilla gradient descent and Adam. The shaded area displays the region between the first and third quartile of 10 runs for different initializations of the network's parameters. Note that GD and Adam are run for 400 times more iterations. \eee
    }
\end{figure}

\begin{figure}
    \centering
    \includegraphics[width=\linewidth]{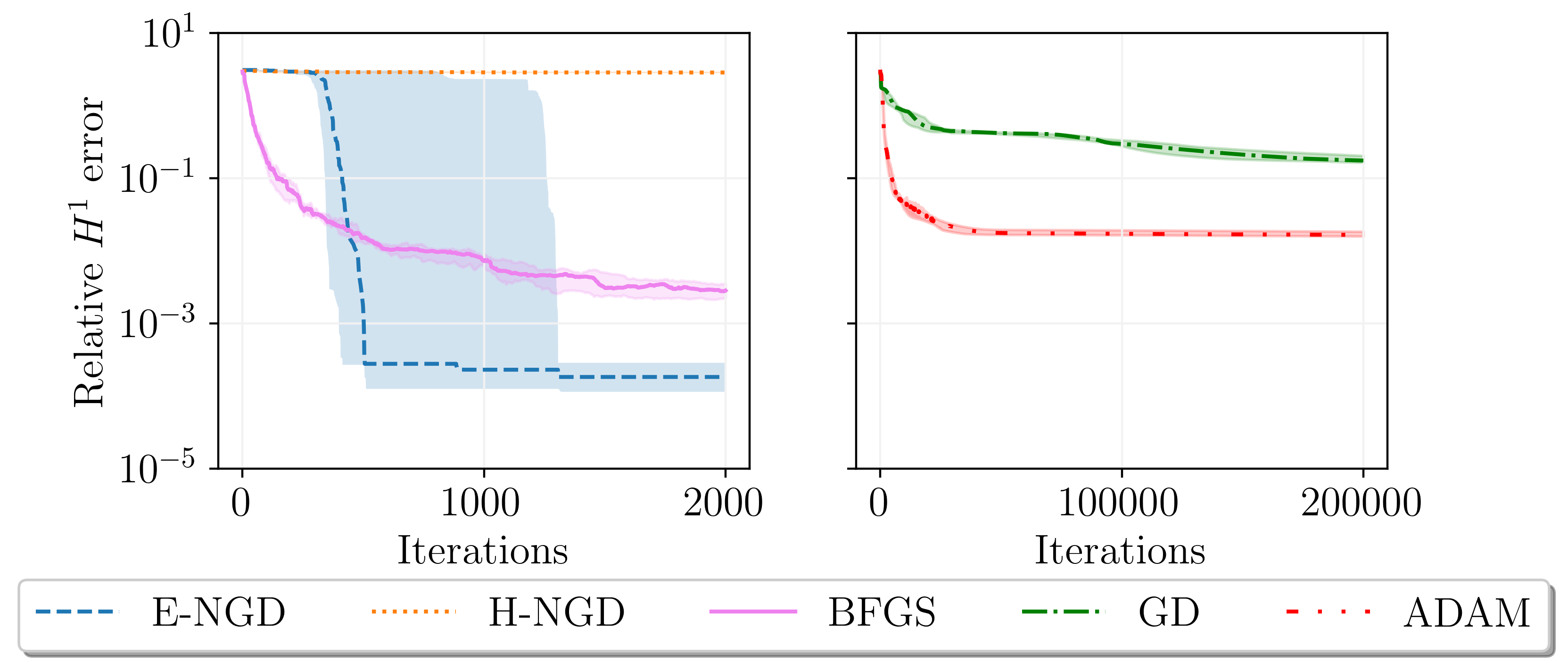}
    \vspace{-.5cm}
    \caption{\mzz The plot shows the median of the relative $H^1$ errors for the heat equation throughout the training process for the five optimizers: energy natural gradient descent, Hilbert natural gradient descent, BFGS, vanilla gradient descent and Adam. The shaded area displays the region between the first and third quartile of 10 runs for different initializations of the network's parameters. Note that GD and Adam are run for 100 times more iterations.\eee
    }
\end{figure}

\begin{figure}
    \centering
    \includegraphics[width=\linewidth]{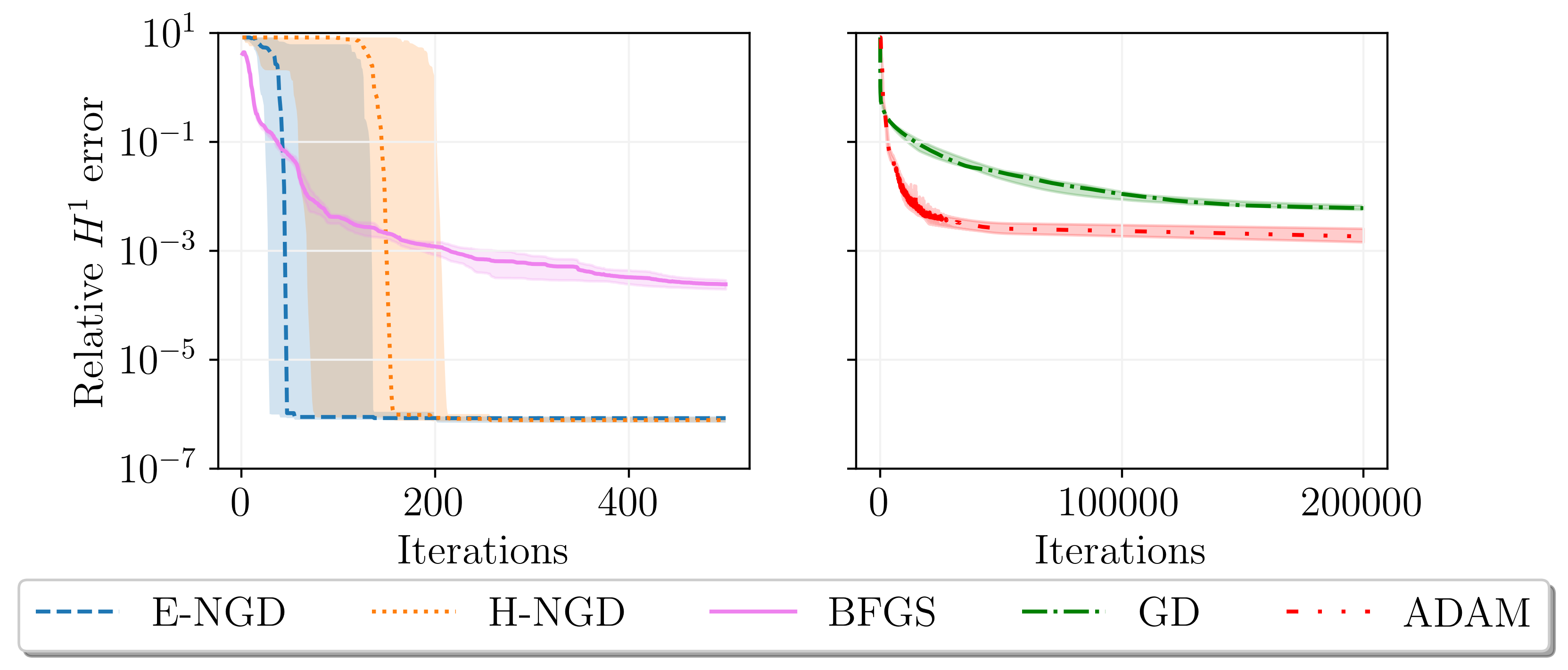}
    \vspace{-.5cm}
    \caption{The plot shows the median of the relative $H^1$ errors for the nonlinear example throughout the training process for the five optimizers: energy natural gradient descent, Hilbert natural gradient descent, BFGS, vanilla gradient descent and Adam. The shaded area displays the region between the first and third quartile of 10 runs for different initializations of the network's parameters. Note that GD and Adam are run for 400 times more iterations.}
\end{figure}

\subsection{Training a Deep Network}
\mzz
    To demonstrate the capability of the energy natural gradient method to be applied to larger and deeper networks, we use the two dimensional Poisson example from section \ref{sec:Poisson_2d} and employ a network with three hidden layers of width 50, which corresponds to roughly 5k trainable weights. We train the network for only 100 iterations. Apart from this we keep all other settings unchanged. We report the results in Table \ref{table:deep_runtimes}. 
\eee

    \begin{table}[h!]
    \begin{center}
    \begin{tabular}{|c|c|c|c|}
    \hline
             & Median $L^2$ Error           & Time Iter         & Time Full    \\
    \hline 
    E-NGD     & $2.5\cdot 10^{-6}$  & $1.1 \textrm{min}$ & $1.8 \textrm{h}$   \\
    \hline
    \end{tabular}
    \caption{\mzz Results for the ENGD training of the deep network for the two dimensional Poisson equation. Reported is the median $L^2$ error over 10 initializations, the time for one iteration and the full optimization time. \eee}\label{table:deep_runtimes}
    \end{center}
    \end{table}
\mzz
We conclude that the energy natural gradient approach can be used for larger and deeper networks, albeit at a higher computational cost. Note however, that we could reduce the amount of iterations required to obtain convergence. The the small networks considered in the main part of the manuscript perform equally well for our tasks at hand, which explains why we did not choose deep networks.
\eee

\end{document}